\documentclass{article}

\usepackage{iclr2027_conference,times}
\renewcommand{\iclrruler}[1]{}

\usepackage{etoolbox}
\makeatletter
\patchcmd{\@maketitle}
  {Anonymous authors\\Paper under double-blind review}
  {\@author}
  {}
  {\PackageError{main}{Could not enable the author list}{Check the ICLR template's title definition.}}
\makeatother

\usepackage{graphicx}
\usepackage[utf8]{inputenc} % allow utf-8 input
\usepackage[T1]{fontenc}    % use 8-bit T1 fonts
\usepackage{hyperref}       % hyperlinks
\usepackage{url}            % simple URL typesetting
\usepackage{booktabs}       % professional-quality tables
\usepackage{amsfonts}       % blackboard math symbols
\usepackage{nicefrac}       % compact symbols for 1/2, etc.
\usepackage{microtype}      % microtypography    
\usepackage{multirow}
\usepackage{amsmath, amssymb, amsthm}

\usepackage{algorithm}
\usepackage{algpseudocode}

\usepackage{graphicx}
\usepackage{wrapfig}
\usepackage{caption}
\usepackage{subcaption}
\usepackage{svg}
\usepackage{marvosym}
\usepackage{soul, color, xcolor}% colors
\usepackage{ulem}
\usepackage{pifont}
\usepackage{utfsym}
\usepackage{xcolor}     % 提供颜色支持
\usepackage{colortbl}   % 提供表格行/列/单元格背景色支持

\definecolor{highlightbg}{HTML}{E8E8FF} 
\definecolor{incpink}{HTML}{D15A9A} 
\newcommand{\diff}[1]{{\scriptsize\textcolor{incpink}{(#1)}}}

\newtheorem*{Theorem*}{Theorem}
\newtheorem{Proposition}{Proposition}
\newtheorem{Corollary}{Corollary}

\title{BridgeGuard: Explicit Safety Drift for Diffusion-Based Autonomous Driving}
\author{%
  Zhenjun Qiu\textsuperscript{1,2}\thanks{Work done during internship at Bosch.}, Jianing Huang\textsuperscript{1}, Dongang Liu\textsuperscript{2}, Baiyu Du\textsuperscript{2}, Yixun Niu\textsuperscript{2},\\
  \textbf{Hao Yang\textsuperscript{1}, Xinyu Huang\textsuperscript{1}, Chuan Hu\textsuperscript{2}\thanks{%
    Corresponding authors: \mbox{Chuan Hu
    (\href{mailto:chuan.hu@sjtu.edu.cn}{\nolinkurl{chuan.hu@sjtu.edu.cn}})} and \\
    \mbox{Shu Liu
    (\href{mailto:shu.liu2@cn.bosch.com}{\nolinkurl{shu.liu2@cn.bosch.com}})}.
  }, Shu Liu\textsuperscript{1}\footnotemark[\value{footnote}]}\\
  \parbox[t]{\dimexpr\textwidth-2\tabcolsep\relax}{\normalfont\small
    \textsuperscript{1}Bosch Research\qquad
    \textsuperscript{2}Shanghai Jiao Tong University
  }
}

\begin{document}
\raggedbottom

\maketitle

\begin{abstract}
  Diffusion-based driving planners capture diverse behaviors but can generate unsafe trajectories under distribution shift. We propose BridgeGuard, a safety-constrained diffusion planning method that progressively strengthens a constraint term during denoising to drive intermediate trajectories toward a scene-dependent safety domain. Corrections operate in a low-dimensional curve space, promoting geometric coherence. A learned module, DistanceFieldNet, predicts a time-dependent distance field from bird's-eye-view features. Value and spatial-gradient supervision at queries sampled beyond expert trajectories teaches this field about both safe and unsafe regions. The learned field supplies the constraint term through safety injection while the pretrained perception backbone and planner remain frozen. We further establish sufficient conditions for terminal safety in an idealized continuous-time bridge. On Bench2Drive, BridgeGuard improves driving score/success rate from 87.99/74.99\% to 90.88/76.36\% for BridgeDrive and from 80.79/58.18\% to 90.46/74.09\% for $\text{DiffusionDrive}^{\text{geo}}$, demonstrating cross-model generalization.
\end{abstract}

\section{Introduction}
Motion planning converts scene observations and route intent into safe,
comfortable, and dynamically feasible trajectories. Classical planners use
search, sampling, or optimization with explicit constraints~\citep{paden2016survey,schwarting2018planning},
while learning-based systems learn from demonstrations~\citep{bansal2018chauffeurnetlearningdriveimitating,hu2023planningoriented}.
Diffusion and flow models support multimodal planning through
conditional generation~\citep{ho2020denoising,lipman2023flow}.
DiffusionDrive denoises multimodal trajectory anchors~\citep{diffusiondrive},
while BridgeDrive refines anchors through diffusion bridges with efficient
probability-flow ODE solvers~\citep{liu2026bridgedrive,zhou2023denoisingdiffusionbridgemodels}.

Learning expert behavior does not ensure safety under limited data coverage or
distribution shift. Existing strategies include cost-gradient guidance during
denoising~\citep{jiang2023motiondiffuser,zhong2023guided},
reward-based search in Diffusion-ES~\citep{yang2024diffusiones},
risk-based candidate ranking in SafeDrive~\citep{safedrive},
and flow-based planning with control barrier functions
(SafeFlow)~\citep{yang2026safeflowmatcher} or learned energies and physical
constraints (GuideFlow)~\citep{liu2026guideflow}.
However, candidate selection depends on proposal coverage, cost guidance on
informative gradients along the sampling path, and barrier constraints on
accurate scene geometry.

Manually Bridged Models (MBM)~\citep{naderiparizi2025constrainedgenerativemodelingmanually}
add the time-weighted negative gradient of a constraint distance to the diffusion
score. This distance is zero on a feasible set $\Omega$ and positive outside it.
Directly applying MBM to driving poses two challenges: the collision-free set depends on the scene and
prediction horizon, requiring future obstacle geometry, and direct waypoint
corrections can distort trajectory geometry.

\begin{figure}[htbp]
    \label{intro}
    \centering
    \includegraphics[width=1\linewidth]{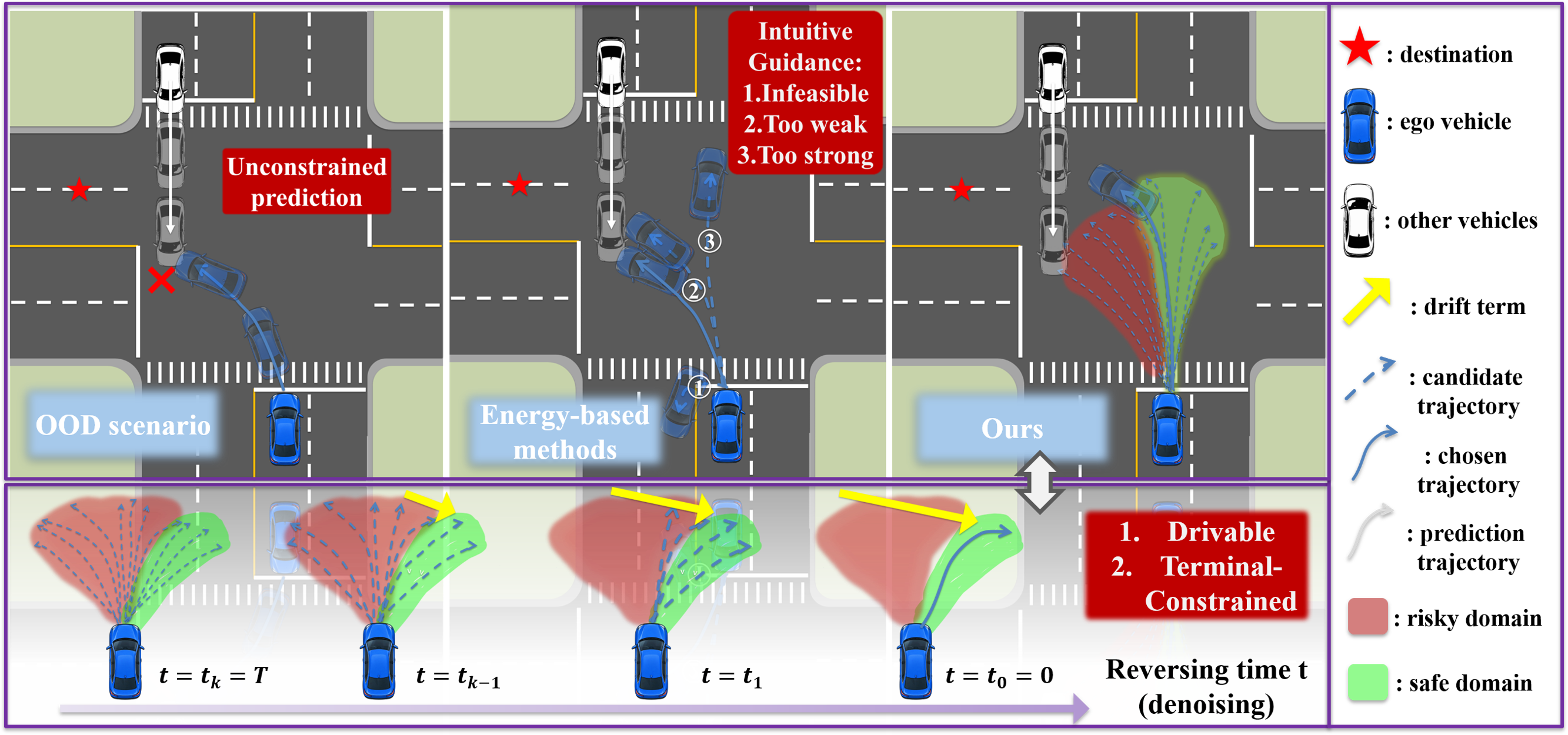}
    \caption{\textbf{BridgeGuard safety correction.} A learned distance field guides curve refinement during denoising, promoting geometric coherence. Terminal safety is guaranteed only for the idealized continuous-time bridge under the stated assumptions.}
    \label{fig:introduction}
\end{figure}

We propose BridgeGuard, which adapts MBM through a learned distance field
and low-dimensional curve corrections. DistanceFieldNet predicts a scene-
and horizon-dependent field from bird's-eye-view (BEV) features, with value and
spatial-gradient supervision covering safe and unsafe regions beyond expert paths.
Its gradients refine curves during denoising to promote geometric coherence,
while the pretrained backbone and planner remain frozen. Our terminal-safety
guarantee concerns an idealized continuous-time bridge and does not directly
extend to the learned field and discretized corrections.

Our contributions are threefold:
\begin{itemize}
    \item We adapt Manually Bridged Model to autonomous driving using a distance function encoding dynamic obstacle constraints and establish sufficient conditions for terminal safety in an idealized continuous-time bridge.
    \item We introduce DistanceFieldNet with value and spatial-gradient supervision beyond expert paths, and inject its learned constraint corrections into a frozen planner through low-dimensional curve updates.
    \item On the Bench2Drive~\citep{jia2024bench2drivemultiabilitybenchmarkingclosedloop} closed-loop benchmark, BridgeGuard improves BridgeDrive's driving score from 87.99 to 90.88 and success rate from 74.99\% to 76.36\%. It also improves $\text{DiffusionDrive}^{\text{geo}}$'s driving score from 80.79 to 90.46 and success rate from 58.18\% to 74.09\%, demonstrating cross-model generalization.
\end{itemize}

\section{Preliminaries}
\label{sec:preliminaries}

% \subsection{Problem Formulation}
% \label{sec:problem-formulation}
% The task of planning in autonomous driving can be formulated as predicting future trajectories of the ego-vehicle based on raw sensor input. We  Our safety correction acts on diffusion bridge model.

\subsection{Diffusion Models in Trajectory Planning}
\label{sec:diffusion-bridge}
\begingroup
\predisplaypenalty=0 % Allow a page break before these equations

\textbf{Diffusion models.}
Diffusion models generate samples by integrating the reverse-time SDE
from $T$ to $0$~\citep{zhou2023denoisingdiffusionbridgemodels}:
\begin{equation}
\mathrm{d}x_t=\left[f(x_t,t)-g^2(t)s(x_t,t)\right]\mathrm{d}t
+g(t)\mathrm{d}\bar{w}_t,
\label{eq:backward-sde}
\end{equation}
\endgroup
where $f$ and $g$ are the forward drift and diffusion coefficients,
$s(x_t,t)=\nabla_{x_t}\log q_t(x_t)$ is the score of the forward marginal,
and $\bar w_t$ is a reverse-time Wiener process.
For the Gaussian perturbation kernel
$q_t(x_t\mid x_0)=\mathcal{N}(\lambda_t x_0,\sigma_t^2I)$, the endpoint
prediction $\hat x_{0\mid t}=x_\theta(x_t,t)$ induces the score
\begin{equation}
\hat s_t=\frac{\lambda_t\hat x_{0\mid t}-x_t}{\sigma_t^2}.
\label{eq:endpoint-score}
\end{equation}
Diffusion bridges condition on an endpoint $x_T$, allowing initialization
from structured priors. For the Gaussian bridge kernel
$q_t(x_t\mid x_0,x_T)=\mathcal{N}(a_t x_T+b_t x_0,c_t^2I)$, the endpoint
prediction similarly induces
$\hat s_t=(a_t x_T+b_t\hat x_{0\mid t}-x_t)/c_t^2$, where $x_T$ is the fixed
terminal endpoint and $a_t,b_t,c_t$ are the bridge-scheduler coefficients~\citep{zheng2025diffusionbridgeimplicitmodels}.
Appendix~\ref{app:diffusion-bridge-background} gives the forward dynamics,
training objectives, and reverse SDE/ODE formulations.

\textbf{Diffusion models for planning.}
Given scene context $z$, the planner predicts $H$ planar waypoints
$x_0\in\mathbb{R}^{2H}$ using temporal or geometric sampling.
Anchors $\mathcal{Y}=\{y_m\}_{m=1}^{M}\subset\mathbb{R}^{2H}$ represent
different trajectory modes.
DiffusionDrive~\citep{diffusiondrive} initializes
noisy candidates around multiple anchors, then refines and ranks them with a
short, truncated denoising process.
% \begin{equation}
% \begin{aligned}
% x_{T_{\mathrm{tr}}}^{(m)}
% &=
% \sqrt{\bar{\alpha}_{T_{\mathrm{tr}}}}\,y_m
% +
% \sqrt{1-\bar{\alpha}_{T_{\mathrm{tr}}}}\,\epsilon_m,
% \quad
% \epsilon_m\sim\mathcal{N}(0,I),
% \\
% \left\{
% \hat{x}_{0\mid t}^{(m)},\,s_t^{(m)}
% \right\}_{m=1}^{M}
% &=
% f_\theta\!\left(
% \left\{x_t^{(m)}\right\}_{m=1}^{M},
% t,z
% \right),
% \\
% m^\star
% &=
% \arg\max_m s_{\mathrm{final}}^{(m)},
% \quad
% \hat{x}_0
% =
% \hat{x}_{0\mid \mathrm{final}}^{(m^\star)}.
% \end{aligned}
% \label{eq:diffusiondrive-prediction}
% \end{equation}
BridgeDrive~\citep{liu2026bridgedrive} uses a mode
classifier $h_\phi$ to select an anchor as $x_T$ and refines it through a
diffusion bridge. BridgeDrive and DiffusionDrive$^{\mathrm{geo}}$ output
$(x^{\mathrm{geo}},v)$: geometric waypoints and a separate target speed;
only the waypoints undergo diffusion. Our correction retains each planner's
waypoint representation and any separately predicted speed, evaluating safety
at physical horizons $\boldsymbol{\tau}$ (Section~\ref{sec:safety-potential}).

\subsection{Constrained Diffusion and Manually Bridged Models}
\label{sec:manual-bridges}

\paragraph{$\Omega$-Bridge.}
For a constraint set $\Omega\subseteq\mathbb{R}^{d}$ and reverse drift
$\nu(x_t,t)=f(x_t,t)-g^2(t)s(x_t,t)$, a vector field
$\mathcal{B}^{\Omega}$ is an \textit{$\Omega$-bridge}~\citep{liu2023learning}
if the modified reverse process
\begin{equation}
\mathrm{d}x_t
=\left[\nu(x_t,t)
-g^2(t)\mathcal{B}^{\Omega}(x_t,t)\right]\mathrm{d}t
+g(t)\mathrm{d}\bar{w}_t
\label{eq:omega-bridge}
\end{equation}
satisfies $\mathbb{P}(x_0\in\Omega)=1$.

\paragraph{Manually bridged models.}
Manually Bridged Models~\citep{naderiparizi2025constrainedgenerativemodelingmanually}
use an \textit{$\Omega$-distance}: a nonnegative, continuous function
$\ell^\Omega$ with finite gradients almost everywhere, satisfying
\begin{equation}
\ell^\Omega(x)=0
\quad\Longleftrightarrow\quad
x\in\Omega.
\label{eq:omega-distance}
\end{equation}
For a schedule $\gamma\in C^1((0,T],\mathbb{R}_{>0})$ satisfying
\begin{equation}
\gamma(T)\approx0,
\qquad
\lim_{t\downarrow 0}\gamma(t)=+\infty,
\label{eq:manual-bridge-schedule}
\end{equation}
the \textit{manual bridge} and the resulting score are
\begin{equation}
b^\Omega(x_t,t)
=-\gamma(t)\nabla_{x_t}\ell^\Omega(x_t),
\qquad
s^\Omega(x_t,t)
=s(x_t,t)+b^\Omega(x_t,t).
\label{eq:manual-bridge}
\end{equation}
This is the score of the tilted density
$q_t^\Omega(x)\propto q_t(x)\exp[-\gamma(t)\ell^\Omega(x)]$,
which downweights constraint violations. The construction alone does not
establish an $\Omega$-bridge; sufficient conditions~\citep{wu2022diffusion}
are given in Appendix~\ref{app:original-bridge-condition}.
Nonnegative component distances can be summed to obtain a manual bridge
for a nonempty intersection of constraint sets.
%Our implementation therefore uses a finite, stabilized schedule and treats Eq.~\eqref{eq:manual-bridge} as safety guidance rather than a hard guarantee.

% \begin{equation}
% b^\Omega(x)=\sum_{j=1}^{J}b^{\Omega_i}(x)
% \label{eq:additive-constraints}
% \end{equation}
% is again a manual bridge. 
% This observation
% underlies the multi-obstacle aggregation in Section~\ref{sec:safety-potential}.

% \input{base/3_proposed_method}
% Alternative to base/3_proposed_method.tex; include only one version.
\section{Method}
\label{sec:method}

We augment a pretrained diffusion-bridge planner with a scene-dependent safety
field. The field represents obstacle-avoidance constraints and supplies spatial
gradients for refining the planner's intermediate predictions. We first define
the safety domain and its violation loss, then describe safety injection through
curve-space correction and trajectory-sampled field supervision. Finally, we
present the architecture that learns the field from sensor features and couples
it to the frozen planner.

\subsection{Safety Domain and distance function}
\label{sec:safety-potential}
\label{sec:Correction process}

\paragraph{Spatial and temporal representation.}
Let $z$ denote the observed driving scene and $x\in\mathbb{R}^{2H}$ the
planner's waypoint sequence, which may use temporal or geometric sampling.
The diffusion state $x_t$ uses the same representation. For safety evaluation,
we map $x$ to positions at $K$ physical prediction horizons
$\boldsymbol{\tau}=(\tau_1,\ldots,\tau_K)$:
\begin{equation}
P=\mathcal{E}_{\boldsymbol{\tau}}(x)
=(\mathbf{p}_1,\ldots,\mathbf{p}_K)\in\mathbb{R}^{2K},
\qquad
\mathbf{p}_i=\mathcal{E}_{\tau_i}(x).
\label{eq:temporal-extrapolation}
\end{equation}
For temporal waypoints already indexed by $\boldsymbol{\tau}$,
$\mathcal{E}_{\boldsymbol{\tau}}$ is the identity; otherwise it evaluates the
trajectory at the query horizons. For geometric plans $(x^{\mathrm{geo}},v)$,
it extrapolates temporal positions along $x^{\mathrm{geo}}$ using the separately
predicted target speed $v$. Such auxiliary predictions are held fixed and
suppressed in the arguments of $\mathcal{E}$. The safety domain and distance function are defined in the space of
time-indexed positions $P$. The diffusion model predicts waypoints in
the base planner's original representation. Safety correction updates
the parameters of a curve fitted to these waypoints, then resamples
the corrected curve into the same waypoint representation.
Physical horizons are measured from the current observation and remain fixed
during correction. They are distinct from diffusion time $t$, which decreases
from $T$ to $0$.

\paragraph{Dynamic obstacle geometry.}
For obstacle $j$, let $\mathbf{c}_j(\tau)$ and $\psi_j(\tau)$ denote its
center and heading at horizon $\tau$, and let $a_j,b_j>0$ be its
half-length and half-width. During training, these quantities are obtained
from obstacle-box and motion annotations, with short-horizon kinematic
extrapolation used to construct the queried future geometry prediction. Thus, an ego
query and an obstacle are compared at the same time horizon. In the
obstacle's local frame, define
\begin{equation}
\begin{bmatrix}\tilde p_j^x\\\tilde p_j^y\end{bmatrix}
=
\begin{bmatrix}
\cos\psi_j(\tau)&\sin\psi_j(\tau)\\
-\sin\psi_j(\tau)&\cos\psi_j(\tau)
\end{bmatrix}
\bigl(\mathbf{p}-\mathbf{c}_j(\tau)\bigr),
\qquad
\Gamma_j(\mathbf{p},\tau)
=\left(\frac{\tilde p_j^x}{a_j}\right)^8
+\left(\frac{\tilde p_j^y}{b_j}\right)^8-1.
\label{eq:gamma_superellipse}
\end{equation}
The eighth-order super-ellipse approximates an oriented obstacle footprint.
The resulting safety domain in temporal-waypoint space is
\begin{equation}
\Omega(z;\boldsymbol{\tau})
=\left\{P\in\mathbb{R}^{2K}\;\middle|\;
\Gamma_j(\mathbf{p}_i,\tau_i)\geq0,
\quad\forall i=1,\ldots,K,\quad\forall j=1,\ldots,J\right\}.
\label{eq:safety-domain}
\end{equation}
Here $J$ is the number of surrounding obstacles. A plan $x$ satisfies these
sampled safety constraints when
$\mathcal{E}_{\boldsymbol{\tau}}(x)\in\Omega(z;\boldsymbol{\tau})$.
% This domain describes
% avoidance of the modeled obstacle regions at the queried horizons; it does
% not encode every driving rule or continuous-time vehicle constraint.

\paragraph{Distance function.}
We assign each temporal query $(\mathbf{p}_i,\tau_i)$ a
per-obstacle pointwise distance function
\begin{equation}
\ell_{j}^\Omega(\mathbf{p}_i,\tau_i;z)
=\left[\arctan\!\left(-\Gamma_j(\mathbf{p}_i,\tau_i)\right)\right]_+,
\qquad [u]_+=\max(u,0).
\label{eq:ell_omega_ij}
\end{equation}
This bounded value is positive inside the obstacle region and exactly
zero on its boundary and exterior. It is continuous and differentiable almost
everywhere with finite gradients. It is an
$\Omega$-distance in the sense of Eq.~\eqref{eq:omega-distance}.

To aggregate obstacles, let $\mathcal{K}(\mathbf{p},\tau;z)$ contain the
indices of the $\min(\kappa,J)$ largest per-obstacle values, with
$\kappa\geq1$. For a fixed query $(\mathbf{p}_i,\tau_i)$, we define the
pointwise distance $d^\Omega(\mathbf{p}_i,\tau_i;z)$
by summing the per-obstacle distances over the retained obstacles
$\mathcal{K}(\mathbf{p}_i,\tau_i;z)$. The trajectory-level distance
then sums these pointwise distances over all query points:
\begin{equation}
\ell^\Omega(P;z,\boldsymbol{\tau})
=\sum_{i=1}^{K}d^\Omega(\mathbf{p}_i,\tau_i;z)
=\sum_{i=1}^{K}
\sum_{j\in\mathcal{K}(\mathbf{p}_i,\tau_i;z)}
\ell_j^\Omega(\mathbf{p}_i,\tau_i;z).
\label{eq:ell_omega_global}
\end{equation}
An empty obstacle set contributes zero. Because all components are
nonnegative and the largest is always retained, the top-$\kappa$ reduction
preserves the zero set of the full sum:
\begin{equation}
\ell^\Omega(P;z,\boldsymbol{\tau})=0
\quad\Longleftrightarrow\quad
P\in\Omega(z;\boldsymbol{\tau}).
\label{eq:trajectory-zero-set}
\end{equation}
For a planner output $x$, the corresponding distance function is therefore
$\ell^\Omega(\mathcal{E}_{\boldsymbol{\tau}}(x);z,\boldsymbol{\tau})$.
%We use $\kappa=1$, so each query measures its most severe obstacle violation.
The aggregation is continuous and differentiable almost everywhere; ranking
switches introduce additional nondifferentiable points. Its negative spatial
gradient supplies a local direction for reducing the distance function.

\subsection{Theoretical Analysis}
\label{sec:theoretical-analysis}
We analyze terminal safety for the idealized continuous-time manual bridge
acting directly in the full temporal-waypoint space, with $X_t\in\mathbb{R}^{2K}$
and safety set $\Omega(z;\boldsymbol{\tau})$. The resulting guarantee applies
to this idealized process and does not directly extend to the approximate
curve-space correction introduced in Section~\ref{sec:safety-injection}.
The following result summarizes Corollary~\ref{cor:app-manual-omega-bridge};
the full assumptions and proof are provided in
Appendix~\ref{app:omega-bridge-conditions}.

\begin{center}
\setlength{\fboxsep}{7pt}
\fbox{%
\begin{minipage}{\dimexpr\linewidth-2\fboxsep-2\fboxrule\relax}
\small
\textbf{Sufficient condition for terminal safety.}
Consider the manually bridged reverse SDE in Eq.~\eqref{eq:manually-bridged-model}. If the distance function satisfies the expected Polyak--\L{}ojasiewicz condition in Eq.~\eqref{eq:app-expected-pl-untouched} and Assumptions A1--A5 in Eqs.~\eqref{eq:app-time-independent-loss}--\eqref{eq:app-local-integrability} of Appendix~\ref{app:terminal-domination-condition} hold, then the manual bridge term $b^\Omega$ is an $\Omega$-bridge, and
\[
\mathbb{P}\!\left(X_0\in\Omega\right)=1.
\]
Thus, the terminal generated trajectory belongs to the prescribed safe set almost surely, as stated in Eq.~\eqref{eq:app-terminal-membership}.
\end{minipage}%
}
\end{center}

The key requirements are expected gradient dominance (Condition~1, Eq.~\eqref{eq:app-expected-pl-untouched}), linking distance function to gradient strength, and terminal domination (A3 and A4, Eqs.~\eqref{eq:app-drift-diffusion-bounds} and~\eqref{eq:app-terminal-attraction}), ensuring schedule strength dominates the drift and diffusion contributions. 
%Condition~1 and the drift bound in A3 are difficult to verify in general because both involve expectations under the manual bridge's marginal law $p_t^\Omega$ (Appendix~\ref{app:interpretation-difficult-conditions}).

\subsection{Safety Injection}
\label{sec:safety-injection}

The goal of safety injection is to refine the planner's waypoints $x$ so that
$\mathcal{E}_{\boldsymbol{\tau}}(x)$ lies in $\Omega(z;\boldsymbol{\tau})$.
We learn a distance field $D^\Omega_\psi(\mathbf{p},\tau;z)$ over spatial
positions and prediction horizons. At inference, we refine the predicted trajectory
in a low-dimensional curve space using field values at temporal query
waypoints. During training, temporal queries from multiple driving modes
teach the field where correction is needed, including locations absent from
the scene's expert trajectory. Detailed training and inference procedures are described in \ref{sec:training-and-inference}.

\paragraph{Correction in a parameterized curve space.}
Direct waypoint corrections can distort path geometry. We therefore fit each
denoised waypoint sequence $\hat x_{0\mid t}\in\mathbb{R}^{2H}$ with parameters
$\boldsymbol{\eta}=\operatorname{Fit}(\hat x_{0\mid t})\in\mathbb{R}^{d_\eta}$,
where $d_\eta\ll2H$. Let $F(\boldsymbol{\eta})$ resample the curve into the
base planner's waypoint representation, following its spatial sampling or
temporal indexing. The learned distance is evaluated at the corresponding
physical query horizons:
\begin{equation}
\widehat{\mathcal{\ell}}_\psi^\Omega(\boldsymbol{\eta};z,\boldsymbol{\tau})
=\sum_{i=1}^{K}
D^\Omega_\psi\!\left(
\mathcal{E}_{\tau_i}(F(\boldsymbol{\eta})),\tau_i;z\right).
\label{eq:curve-safety-objective}
\end{equation}
We differentiate through $\mathcal{E}$ and $F$, holding the query horizons,
sampling convention, and any auxiliary predictions fixed.
After updating $\boldsymbol{\eta}$, the corrected endpoint is
$\hat x^\Omega_{0\mid t}=F(\boldsymbol{\eta}^\Omega)$.
Thus, safety is evaluated in temporal-waypoint space, while correction acts
on the curve parameters. Please see more details on inference process in \ref{sec:training-and-inference}. The induced waypoint correction is
restricted to directions expressible by the chosen curve family; its
first-order form is derived in
Appendix~\ref{Low-dimensional constrained correction}.

% \paragraph{From score guidance to endpoint correction.}
% For a diffusion model with Gaussian perturbation kernel
% $q_t(x_t\mid x_0)=\mathcal{N}(\lambda_t x_0,\sigma_t^2I)$ and endpoint
% prediction $\hat x_{0\mid t}=x_\theta(x_t,t,z)$, the induced score is
% \begin{equation}
% \hat s_t=\frac{\lambda_t\hat x_{0\mid t}-x_t}{\sigma_t^2}.
% \label{score function}
% \end{equation}
% At fixed $x_t$ and $z$, a score increment $\Delta s_t$ is equivalently written
% as an endpoint increment:
% \begin{equation}
% \hat x_{0\mid t}^{\Omega}
% =\hat x_{0\mid t}+\frac{\sigma_t^2}{\lambda_t}\Delta s_t,
% \qquad
% \Delta s_t=\frac{\lambda_t}{\sigma_t^2}
% \bigl(\hat x_{0\mid t}^{\Omega}-\hat x_{0\mid t}\bigr).
% \label{direct correction}
% \end{equation}
% This local identity holds at interior scheduler times with
% $\lambda_t,\sigma_t>0$. For the Gaussian bridge kernel
% $q_t(x_t\mid x_0,x_T)=\mathcal{N}(a_t x_T+b_t x_0,c_t^2I)$, it becomes
% $\hat x_{0\mid t}^{\Omega}=\hat x_{0\mid t}+(c_t^2/b_t)\Delta s_t$, or
% $\Delta s_t=(b_t/c_t^2)(\hat x_{0\mid t}^{\Omega}-\hat x_{0\mid t})$
% \citep{zheng2025diffusionbridgeimplicitmodels}, as detailed in
% Appendix~\ref{discussion of correction equation}.

\paragraph{Supervision beyond the expert path.}
The learned distance function $D^\Omega_\psi(\mathbf{p},\tau;z)$ must provide useful values and gradients wherever the planner may
query it, whereas expert waypoints cover only a narrow and typically safe part
of this domain. We thus leverage additional sources, such as trajectories sampled from a pretrained planner or trajectory anchors, including the clustered anchors used for training~\ref{fig:training}. 
For each training scene $z$, we map sampled plans through
$\mathcal{E}_{\boldsymbol{\tau}}$ and sample the resulting
temporal waypoint--horizon pairs $(\mathbf{p}_i,\tau_i)$.
We denote this distribution of sampled points and time horizon as $(\mathbf{p},\tau)\sim \mathcal{Q}(z)$.
These queries cover both safe and dangerous regions~\citep{luo2026unleashingvlapotentialsautonomous,bansal2018chauffeurnetlearningdriveimitating} without requiring
additional expert demonstrations. The sampling configuration is given in
Section~\ref{experiments}. Quantitative improvements are reported in Sec.~\ref{sec: Ablation Study on the Training Strategy}.
\begin{figure}[htbp]
    \centering
    \includegraphics[width=1\linewidth]{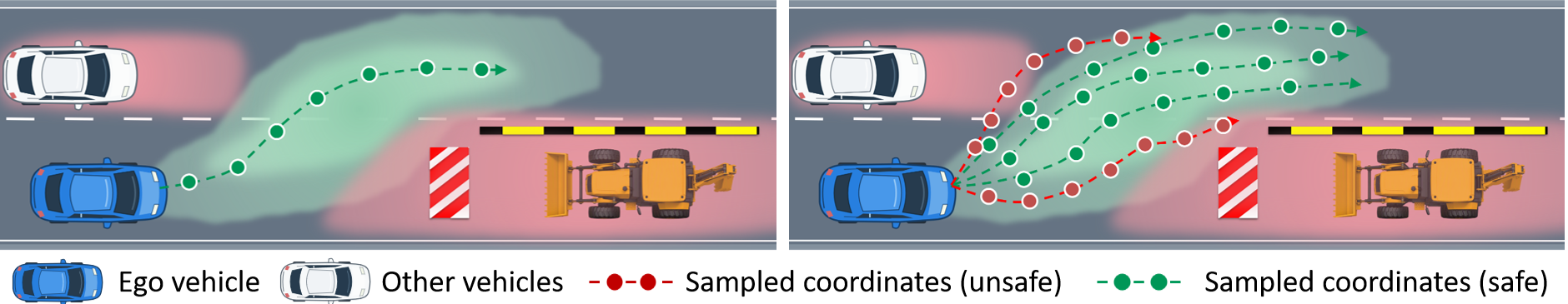}
    \caption{\textbf{DistanceFieldNet training queries.} Expert-only sampling (left) covers a narrow, typically safe region. Anchor-based sampling (right) broadens coverage to safe (green) and violating (red) queries beyond the expert path.}
    \label{fig:training}
\end{figure}

\subsection{Training Objective and Inference}
\label{sec:training-and-inference}
\paragraph{Training objective.}
We train the distance function with the following value and gradient ~\citep{li2023boostingmultimodalmodelperformance,peng2022balancedmultimodallearningonthefly,rakib2025g2dboostingmultimodallearning} supervision from distance function
\begin{align}
\mathcal{L}_{\mathrm{D}}
&=\mathbb{E}_{z,(\mathbf{p},\tau)\sim\mathcal{Q}(z)}
\left[\left|D^\Omega_\psi(\mathbf{p},\tau;z)-d^\Omega(\mathbf{p},\tau;z)\right|^2\right],
\label{eq:field-alignment}\\
\mathcal{L}_{\mathrm{GD}}
&=\mathbb{E}_{z,(\mathbf{p},\tau)\sim\mathcal{Q}(z)}
\left[\left\|\nabla_{\mathbf{p}}
D^\Omega_\psi(\mathbf{p},\tau;z)
-\nabla_{\mathbf{p}}
d^\Omega(\mathbf{p},\tau;z)\right\|_2^2\right].
\label{eq:gradient-distillation}
\end{align}
At nondifferentiable target locations, a consistent subgradient convention is used.
The two supervision terms give the overall training loss
\begin{equation}
\mathcal{L}
=\mathcal{L}_{\mathrm{D}}
+\lambda_{\mathrm{gd}}\mathcal{L}_{\mathrm{GD}}.
\label{eq:field-training-objective}
\end{equation}
During training, only parameters $\psi$ are optimized; the original perception backbone and
diffusion planner remain frozen.

\paragraph{Inference.}
DistanceFieldNet predicts the field from observed BEV features and
spatial--temporal queries.
We use a finite, stabilized score-level schedule $\gamma_\varepsilon(t)$
that is small near $T$ and increases toward $t=0$.
At fixed $x_t$ and $z$, Eq.~\eqref{eq:endpoint-score} gives
\begin{equation}
\hat x_{0\mid t}^\Omega
=\hat x_{0\mid t}+\frac{\sigma_t^2}{\lambda_t}\Delta s_t,
\qquad
\Delta s_t=\hat s_t^\Omega-\hat s_t.
\label{direct correction}
\end{equation}
At each reverse step, we fit
$\boldsymbol{\eta}=\operatorname{Fit}(\hat x_{0\mid t})$ and apply
\begin{equation}
\boldsymbol{\eta}^\Omega
=\boldsymbol{\eta}-\alpha(t)W_\eta
\nabla_{\boldsymbol{\eta}}\widehat{\mathcal{\ell}}_\psi^\Omega
(\boldsymbol{\eta};z,\boldsymbol{\tau}),
\qquad
\alpha(t)=\frac{\sigma_t^2}{\lambda_t}\gamma_\varepsilon(t).
\label{eq:curve-correction}
\end{equation}
Here $W_\eta$ is positive diagonal and sets the relative parameter scales.
Query horizons and auxiliary predictions remain fixed, while positions are
recomputed through $\mathcal{E}_{\boldsymbol{\tau}}\circ F$.
For Gaussian diffusion bridges, $\alpha(t)=(c_t^2/b_t)\gamma_\varepsilon(t)$.
The effective step is the full product $\alpha(t)$; conversion ratios are
evaluated at nonsingular interior times, with endpoint limits handled by the
base scheduler.

We resample $\hat x_{0\mid t}^\Omega=F(\boldsymbol{\eta}^\Omega)$ and use its
induced score in the base reverse SDE or probability-flow ODE.
For diffusion bridges, we keep $x_T$ fixed and retain the base solver's
anchor-conditioning drift. Algorithm~\ref{alg:safe_bridgedrive} summarizes the procedure;
Appendix~\ref{discussion of correction equation} derives the score--endpoint
conversion.

\paragraph{Scope of the bridge interpretation.}
Equation~\eqref{direct correction} is exact, but curve fitting, resampling,
and restricted gradients (Eq.~\eqref{eq:curve-pullback}) approximate full-state
manual-bridge dynamics. The idealized guarantee in
Appendix~\ref{app:omega-bridge-conditions} does not directly extend to learned
fields, finite schedules, or discretized updates. The curve representation
promotes geometric coherence without certifying kinematic feasibility or
collision freedom.

\begin{algorithm}[t]
\caption{Diffusion Planning with Learned Safety Injection (on BridgeDrive)}
\label{alg:safe_bridgedrive}
\small
\begin{algorithmic}[1]
\Require Scene $z$, anchors $\mathcal{Y}=\{y_m\}$, pretrained planner
$(h_\phi,x_\theta)$, separately predicted speed $v$, trained DistanceFieldNet $D_\psi^\Omega$, horizons $\boldsymbol{\tau}$,
reverse times $T=t_N>\cdots>t_0=0$
\Ensure Refined geometric path $x_0$ with unchanged target speed $v$
\State $m^\star\gets\arg\max_m h_\phi(y_m,z)$;
$x_{t_N}\gets x_T\gets y_{m^\star}$
\For{each reverse step $t=t_n\to t'=t_{n-1}$, $n=N,\ldots,1$}
    \State $\hat x_{0\mid t}
    \gets x_\theta(x_t,t,x_T,z)$ \Comment{Predict the clean path}
    \State $\boldsymbol{\eta}\gets\operatorname{Fit}(\hat x_{0\mid t})$
    \Comment{Project the path into curve space}
    \State Evaluate $\nabla_{\boldsymbol{\eta}}\widehat{\mathcal{\ell}}_\psi^\Omega$
    using $\mathcal{E}_{\boldsymbol{\tau}}\circ F$ in Eq.~\eqref{eq:curve-safety-objective}
    \Comment{Compute the safety direction}
    \State Update $\boldsymbol{\eta}^\Omega$ by Eq.~\eqref{eq:curve-correction}
    with $\alpha(t)=(c_t^2/b_t)\gamma_\varepsilon(t)$
    \Comment{Correct the curve coefficients}
    \State $\hat x_{0\mid t}^\Omega
    \gets F(\boldsymbol{\eta}^\Omega)$ \Comment{Recover refined waypoints}
    \State $x_{t'}\gets\operatorname{ODESolverStep}
    (x_t,\hat x_{0\mid t}^\Omega,x_T,t,t')$ \Comment{Advance the reverse ODE}
\EndFor
\State \Return $(x_0,v)$
\end{algorithmic}
\end{algorithm}

% \begin{algorithm}[t]
% \caption{Diffusion-Bridge Planning with a Learned Safety Field}
% \label{alg:safe_bridgedrive}
% \small
% \begin{algorithmic}[1]
% \Require Scene $z$, anchors $\mathcal{Y}$, pretrained classifier $h_\phi$
% and denoiser $x_\theta$, field $D_\psi$, physical horizons $\boldsymbol{\tau}$,
% reverse-step schedule
% \Ensure Refined path and target speed $x_0$
% \State Extract $\mathbf{F}_{\mathrm{BEV}}(z)$ with the frozen perception backbone
% \State $m^\star\gets\arg\max_m h_\phi(y_m,z)$;
% $x_T\gets y_{m^\star}$
% \For{each reverse solver step from $t$ to $t'<t$}
%     \State $(\hat x^{\mathrm{geo}}_{0\mid t},\hat v_{0\mid t})\gets x_\theta(x_t,t,x_T,z)$
%     \State $\boldsymbol{\eta}\gets\operatorname{Fit}(\hat x^{\mathrm{geo}}_{0\mid t})$;
%     $\mathbf{u}\gets$ original geometric $x$-coordinates
%     \State Extrapolate $P\gets\mathcal{E}_{\boldsymbol{\tau}}
%     (F(\boldsymbol{\eta};\mathbf{u}),\hat v_{0\mid t})$
%     \State Evaluate $\widehat{\mathcal{\ell}}_\psi^\Omega$ and its
%     parameter gradient through the extrapolation in
%     Eq.~\eqref{eq:curve-safety-objective}
%     \State Update $\boldsymbol{\eta}^\Omega$ using
%     Eq.~\eqref{eq:curve-correction} with
%     $\alpha(t)=(c_t^2/b_t)\gamma_\varepsilon(t)$
%     \State Form $\hat x_{0\mid t}^\Omega$ using
%     Eq.~\eqref{eq:corrected-endpoint-score}
%     \State $x_{t'}\gets\operatorname{BridgeSolverStep}
%     (x_t,\hat x_{0\mid t}^\Omega,x_T,t,t')$
% \EndFor
% \State \Return $x_0$
% \end{algorithmic}
% \end{algorithm}

\subsection{Model Architecture}
\label{sec:model-architecture}
\begin{figure}[htbp]
    \centering
    \includegraphics[width=1.0\linewidth]{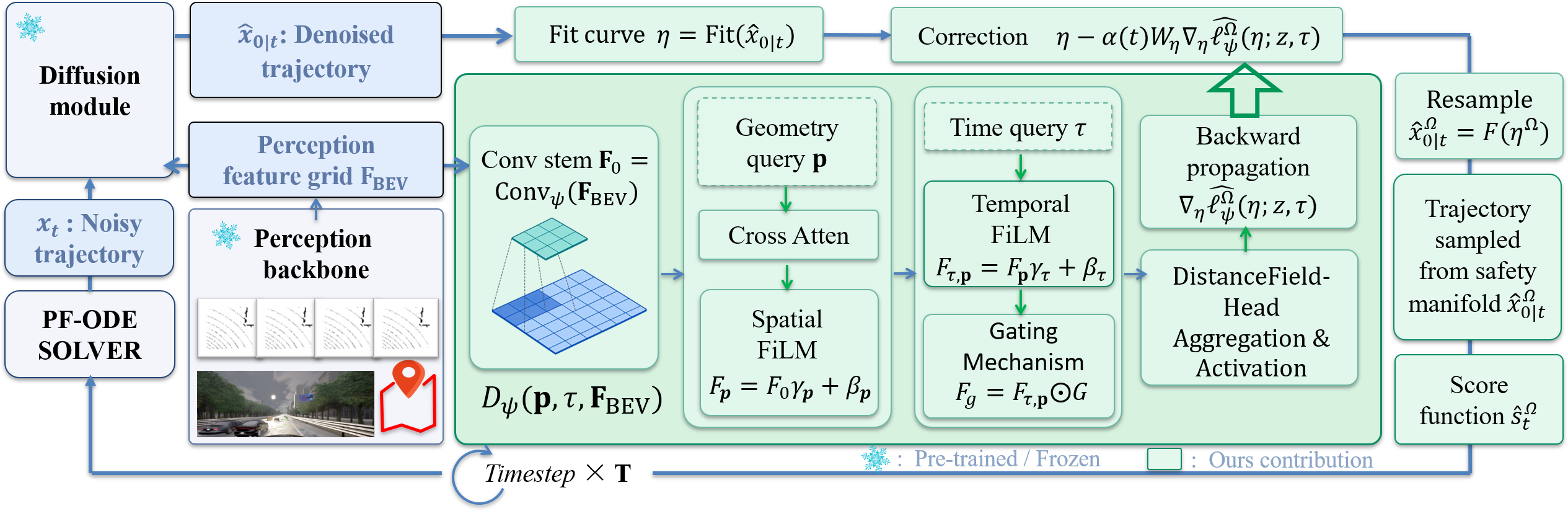}
    \caption{\textbf{BridgeGuard inference.} Green components denote BridgeGuard additions. At each denoising step, DistanceFieldNet supplies gradients for curve correction; resampled waypoints are returned to the frozen planner's solver.}
    \label{fig:placeholder}
\end{figure}

\paragraph{Perception and base planner.}
The system comprises a perception backbone, an anchor-based diffusion
planner, and the learned distance-field module, \emph{DistanceFieldNet}. Our
BridgeDrive implementation retains the pretrained TransFuser++ perception
backbone~\citep{Jaeger2023ICCV}. Its pretraining includes predicting other agents'
velocities, so its BEV features encode motion information for DistanceFieldNet.
The correction can also be coupled to
DiffusionDrive$^{\mathrm{geo}}$, the geometric adaptation of DiffusionDrive
described in the BridgeDrive paper~\citep{diffusiondrive,liu2026bridgedrive}.
The backbone and selected planner remain frozen. Sensor observations are
encoded into $\mathbf{F}_{\mathrm{BEV}}(z)$, which is shared by the planner
and DistanceFieldNet. We retain each planner's denoising and trajectory-selection
procedure, while DistanceFieldNet supplies the safety objective for refining its
clean waypoint predictions at each reverse step. 
% Resampling follows the
% selected planner's representation. For BridgeDrive, the separately predicted
% target speed is used for temporal extrapolation and control, and is neither
% diffused nor corrected by our method.

\paragraph{Query-conditioned DistanceFieldNet.}
DistanceFieldNet takes a waypoint sequence $P=(\mathbf{p}_1,\ldots,\mathbf{p}_K)$,
the corresponding physical horizons $\boldsymbol{\tau}=(\tau_1,\ldots,\tau_K)$,
and shared BEV features $\mathbf{F}_{\mathrm{BEV}}(z)$ as input.
A shared convolutional stem encodes the BEV features as
$\mathbf{F}_0=\mathrm{Conv}_\psi(\mathbf{F}_{\mathrm{BEV}})$.
We describe the per-waypoint computation for one query pair
$(\mathbf{p},\tau)$, omitting the waypoint index below.
Sinusoidal positional features~\citep{NIPS2017_3f5ee243} and MLPs embed the
BEV-normalized coordinate $\bar{\mathbf{p}}$ and horizon $\tau$ as
$\mathbf{e}_{\mathbf{p}}$ and $\mathbf{e}_{\tau}$.
The spatial query embedding cross-attends to position-encoded BEV tokens
$\mathbf{Z}_{\mathrm{BEV}}$, yielding a context $\mathbf{h}$ that generates
spatial FiLM~\citep{perez2017filmvisualreasoninggeneral} parameters. Spatial FiLM modulates $\mathbf{F}_0$ into
$\mathbf{F}_{\mathbf{p}}$, and temporal FiLM driven by $\mathbf{e}_{\tau}$
further produces $\mathbf{F}_{\tau,\mathbf{p}}$.
A query-specific gate filters the conditioned features as
$\mathbf{F}_{g}=\mathbf{F}_{\tau,\mathbf{p}}\odot G$.
DistanceFieldHead aggregates $\mathbf{F}_{g}$ into a dense response map
$\mathbf{R}$, reduces the flattened map to a scalar logit $r$, and applies
a bounded monotonic activation.
% \begin{equation}
% D_\psi(\mathbf{p},\tau,\mathbf{F}_{\mathrm{BEV}})
% =
% \frac{\pi}{2}-\arctan(\operatorname{softplus}(r)).
% \end{equation}
The same network evaluates all $K$ query pairs, and their scalar outputs
are summed to obtain the trajectory-level safety objective.
% :
% \begin{equation}
% \widehat\ell_\psi^\Omega(P;z,\boldsymbol{\tau})
% =\sum_{i=1}^{K}
% D_\psi(\mathbf{p}_i,\tau_i,\mathbf{F}_{\mathrm{BEV}}(z)).
% \end{equation}
Backpropagating this sum through the waypoint coordinates and trajectory
parameterization provides the curve-space correction direction.
Architecture details are given in Appendix~\ref{neutral network}.

\section{Experiments}\label{experiments}
\subsection{CLOSED-LOOP EVALUATION IN CARLA Simulator}
\paragraph{Benchmark}
% 在本文中，我们主要在模拟器carla\cite{Dosovitskiy17}上进行算法的闭环实验验证，闭环评测更贴近真实情景，能够好地暴露安全失效问题。Bench2Drive is a well known CARLA benchmark proposed by the paper~\cite{jia2024bench2drivemultiabilitybenchmarkingclosedloop} consisting of 220 very short (~150m) routes split across all towns with 1 safety critical scenario in each route. It's constructed under the CARLA Leaderboard 2.0\cite{Hidden Biases of End-to-End Driving Datasets} protocol for end-to-end autonomous driving.许多优秀的端到端或者VLM算法已经在该基准测试上取得了优秀的评测结果（\cite{liu2026bridgedrive, Jaeger2023ICCV,fu2025orionholisticendtoendautonomous}）。但是我们在复现这些算法实验的时候注意到，结果中仍然存在大量的安全失效场景亟待改善，这也是本文工作希望改进的方向。我们在此基准数据集上与当前的sota算法进行了对比。
We evaluate BridgeGuard on Bench2Drive~\citep{jia2024bench2drivemultiabilitybenchmarkingclosedloop},
a closed-loop benchmark in CARLA~\citep{Dosovitskiy17} built on the CARLA
Leaderboard 2.0 protocol. Its 220 routes span CARLA towns, each covering
approximately 150\,m and one safety-critical scenario. This setting exposes
failures arising from interactions with the environment. Motivated by safety
failures observed in our reproductions of leading end-to-end and vision-language-action (VLA) planners~\citep{liu2026bridgedrive,Jaeger2023ICCV,fu2025orion,Renz2025cvpr},
we compare their performance with BridgeGuard on these challenging scenarios.

\paragraph{Dataset}
Following BridgeDrive~\citep{liu2026bridgedrive}, we use the
publicly released CARLA Garage training dataset~\citep{Zimmerlin2024ArXiv}.
The data were generated using the PDM-Lite expert~\citep{beisswenger2024pdm, Sima2024ECCV} and include
camera images, LiDAR observations, and expert driving annotations.
Table~\ref{tab:main experiment results} also includes baselines trained on the official
Bench2Drive dataset, generated using the Think2Drive
expert~\citep{li2024think2driveefficientreinforcementlearning}.

\paragraph{Baselines}
We compare BridgeGuard with ten end-to-end driving baseline configurations in
Table~\ref{tab:main experiment results}. DriveTransformer~\citep{jia2025drivetransformer},
TransFuser++~\citep{Zimmerlin2024ArXiv}, and
SparseDriveV2~\citep{sun2026sparsedrivev2} provide non-VLA, non-diffusion
reference points; ORION~\citep{fu2025orion}, MindDrive~\citep{fu2025minddrive},
and SimLingo~\citep{Renz2025cvpr} represent VLA-based planners;
DiffusionDrive~\citep{diffusiondrive} and
BridgeDrive~\citep{liu2026bridgedrive} represent
diffusion-based planners; and AutoMoT~\citep{huang2026automot} combines VLA and
diffusion-based planning. We report both DiffusionDrive adaptations from
BridgeDrive~\citep{liu2026bridgedrive}:
DiffusionDrive$^{\mathrm{temp}}$ retains the original temporal waypoint
representation, while DiffusionDrive$^{\mathrm{geo}}$ predicts geometric
waypoints and a separate target speed. Our DiffusionDrive-based results use
the geometric variant. The table also identifies the data-collection expert
for each method.

\paragraph{Implementation Details}
We freeze the pretrained perception backbone and diffusion planner and train
only DistanceFieldNet. For each training scene, we sample 20 query points from two anchors and supervise field values and spatial gradients using
annotated obstacle geometry. Safety is evaluated over a 0.5~s horizon.
We fit planar paths with $y(x)=\eta_1x+\eta_2x^2+\eta_4x^4$,
where $x$ and $y$ denote forward and lateral ego-frame coordinates.
This quartic family has three free coefficients, with constant and cubic
terms fixed at zero. Correction updates these coefficients while keeping
the target speed fixed. At inference, DistanceFieldNet uses observed
BEV features and waypoint--horizon queries, without explicit obstacle dimensions
or ground-truth geometry. Detailed network, supervision, curve-refinement,
training, and sensor settings are given in Appendix~\ref{app:implementation-config}.
\begin{table}[htbp]
\centering
\caption{Comparison of different methods on Bench2Drive.}
\label{tab:main experiment results}
\resizebox{\columnwidth}{!}{%
\begin{tabular}{lccc c c}
\toprule
Method & Expert & VLA & Diffusion & DS & SR(\%) \\
\midrule
% TCP-traj*~\citep{wu2022trajectoryguided} & Think2Drive & \textcolor{red}{\ding{56}} & \textcolor{red}{\ding{56}} & 59.90 & 30.00 \\
% UniAD-Base~\citep{hu2023_uniad} & Think2Drive & \textcolor{red}{\ding{56}} & \textcolor{red}{\ding{56}} & 45.81 & 16.36 \\
% VAD~\citep{jiang2023vad} & Think2Drive & \textcolor{red}{\ding{56}} & \textcolor{red}{\ding{56}} & 42.35 & 15.00 \\
DriveTransformer~\citep{jia2025drivetransformer} & Think2Drive & \textcolor{red}{\ding{56}} & \textcolor{red}{\ding{56}} & 63.46 & 35.01 \\
ORION~\citep{fu2025orion} & Think2Drive & \textcolor{teal}{\usym{2714}} & \textcolor{red}{\ding{56}} & 77.74 & 54.62 \\
MindDrive~\citep{fu2025minddrive} & Think2Drive & \textcolor{teal}{\usym{2714}} & \textcolor{red}{\ding{56}} & 80.59 & 58.26 \\
TransFuser++~\citep{Zimmerlin2024ArXiv} & PDM-Lite & \textcolor{red}{\ding{56}} & \textcolor{red}{\ding{56}} & 84.21 & 67.27 \\
% ReCogDrive~\citep{li2025recogdrive} & PDM-Lite & \textcolor{teal}{\usym{2714}} & \textcolor{teal}{\usym{2714}} & 84.21 & 67.27 \\
SimLingo~\citep{Renz2025cvpr} & PDM-Lite & \textcolor{teal}{\usym{2714}} & \textcolor{red}{\ding{56}} & 85.07 & 67.27 \\
AutoMoT~\citep{huang2026automot}  & PDM-Lite & \textcolor{teal}{\usym{2714}} & \textcolor{teal}{\usym{2714}} & 89.42 & 74.09 \\
SparseDriveV2~\citep{sun2026sparsedrivev2} & PDM-Lite & \textcolor{red}{\ding{56}} & \textcolor{red}{\ding{56}} & 89.15 & 70.00 \\
DiffusionDrive$^{\mathrm{temp}}$~\citep{liu2026bridgedrive} & PDM-Lite & \textcolor{red}{\ding{56}} & \textcolor{teal}{\usym{2714}} & 77.68 & 52.72 \\
\rowcolor{highlightbg}DiffusionDrive$^{\mathrm{geo}}$~\citep{liu2026bridgedrive} & PDM-Lite & \textcolor{red}{\ding{56}} & \textcolor{teal}{\usym{2714}} & 80.79 & 58.18 \\
\rowcolor{highlightbg}BridgeDrive~\citep{liu2026bridgedrive} & PDM-Lite & \textcolor{red}{\ding{56}} & \textcolor{teal}{\usym{2714}} & $87.99$ & $74.99$ \\
\rowcolor{highlightbg}\textbf{Ours (based on DiffusionDrive$^{\mathrm{geo}}$)} & PDM-Lite & \textcolor{red}{\ding{56}} & \textcolor{teal}{\usym{2714}} & \textbf{90.46}\scriptsize{\textcolor{magenta!90}{ \diff{+9.67}}} & {74.09}\scriptsize{\textcolor{magenta!90}{ \diff{+15.91}}} \\
\rowcolor{highlightbg}\textbf{Ours (based on BridgeDrive)} & PDM-Lite & \textcolor{red}{\ding{56}} & \textcolor{teal}{\usym{2714}} & \textbf{90.88}\scriptsize{\textcolor{magenta!90}{ \diff{+2.89}}} & \textbf{76.36}\scriptsize{\textcolor{magenta!90}{ \diff{+1.37}}} \\
\bottomrule
\end{tabular}
}
\end{table}
% https://arxiv.org/pdf/2606.25736 87分 统一时序扩散
\paragraph{Main results}
% 表格~\ref{tab:main experiment results}中展示了我们的工作和近两年优秀的端到端以及VLA的自动驾驶工作在Bench2Drive基准上的对比。可以看到，我们的工作在Driving Score (DS) and Success Rate (SR)上都取得了sota水平，超越了表中列出的所有基线。值得注意的是，我们的工作对比VLA架构的算法在更小的模型体积上仍有优势，对比同样使用扩散模型作为规划模块的模型也表现出了更好的性能。本文的工作可以适用于普遍的扩散模型，基于BridgeDrive和DiffusionDrive开发，都相较于基线展现出了明显的提升（相较于BridgeDrive得分和成功率分别提高了\hl{2.89和1.37%}，在DiffusionDrive工作的基础上运用本方法，得分和成功率分别提高了\hl{9.67和15.91%}），显示了算法的有效性和鲁棒性。
Table~\ref{tab:main experiment results} shows that BridgeGuard on BridgeDrive
achieves the highest Driving Score (DS, 90.88) and Success Rate (SR, 76.36\%)
among the compared methods, outperforming VLA baselines with a much
smaller model. BridgeGuard improves DS by 2.89 and SR by 1.37\% over BridgeDrive,  9.67 and 15.91\%
over DiffusionDrive$^{\mathrm{geo}}$. These gains support its effectiveness
across both diffusion planners. The larger improvements on DiffusionDrive$^{\mathrm{geo}}$ may
reflect its higher baseline collision rate, which leaves more room for safety
correction.
When trained on LEAD, BridgeGuard reaches 96.76 DS and 89.78\% SR,
outperforming the compared baselines (Appendix~\ref{app:lead-results}).

\subsection{Ablation Study and Qualitative Analysis}

\paragraph{Curve dimensionality.}
The quadratic ablation fixes $\eta_4=0$, leaving the two coefficients of
$y(s)=\eta_1s+\eta_2s^2$. Allowing the third coefficient $\eta_4$ gives our
quartic curve and improves DS from 88.65 to 90.88 and SR from 73.18\% to
76.36\% (Table~\ref{tab:different curve dimensionality}).
Mean ability also rises from 64.77\% to 69.74\%, despite a small decline in
merging performance (75.00\% to 73.75\%). These results favor the richer curve
representation for safety correction; full results are in
Appendix~\ref{app:curve-dimensionality}.

\paragraph{Training-query sampling.}
\label{sec: Ablation Study on the Training Strategy}
Anchor-based queries outperform expert-only queries at the same ten-point
budget (Table~\ref{tab:different training strategies}). Sampling 20 points from
two anchors yields the best mean DS (90.96) and SR (75.23\%), while increasing
to 100 points from ten anchors provides no further improvement. These results
support supervising the distance field beyond expert trajectories with a
moderate query budget.

% Keep the ablation tables on a text page rather than a float-only page.
\begin{table}[!htb]
\centering
\begin{minipage}[t]{0.48\linewidth}
\vspace{0pt}
\centering
\caption{Curve dimensionality.}
\label{tab:different curve dimensionality}
\small
\setlength{\tabcolsep}{4pt}
\begin{tabular}{@{}lrr@{}}
\toprule
Metric & Quadratic & Quartic \\
\midrule
DS & 88.65 & \textbf{90.88} \\
SR (\%) & 73.18 & \textbf{76.36} \\
Merging (\%) & \textbf{75.00} & 73.75 \\
Mean ability (\%) & 64.77 & \textbf{69.74} \\
\bottomrule
\end{tabular}
\end{minipage}\hfill
\begin{minipage}[t]{0.50\linewidth}
\vspace{0pt}
\centering
\caption{Training-query sampling.}
\label{tab:different training strategies}
\small
\setlength{\tabcolsep}{2.5pt}
\begin{tabular}{@{}r l c c@{}}
\toprule
Queries & Source & DS & SR (\%) \\
\midrule
10 & Expert & $89.17\pm0.68$ & $70.23\pm2.25$ \\
10 & 1 anchor & $89.67\pm0.23$ & $72.73\pm1.93$ \\
\rowcolor{blue!10}
20 & 2 anchors & $\mathbf{90.96\pm0.34}$ & $\mathbf{75.23\pm0.96}$ \\
100 & 10 anchors & $90.69\pm0.01$ & $73.41\pm0.96$ \\
\bottomrule
\end{tabular}
\end{minipage}
\end{table}

\paragraph{Advantages over BridgeDrive}
Figures~\ref{fig:detail_carla} and \ref{fig:detail_carla2} show traffic conflicts
that BridgeDrive fails to resolve during an unprotected left turn and a
multi-lane change on Bench2Drive. BridgeGuard instead yields and completes both maneuvers
safely. These cases illustrate how our safety drift term corrects potentially
dangerous trajectories, pulling the ego vehicle away from imminent conflicts
and back into the safe region in an explicit manner. We refer readers to the
supplementary videos described in Appendix~\ref{app:supplementary-videos}
for comparisons of BridgeDrive and BridgeGuard in two lane-changing cases
and one unprotected left-turn case.
% Its distance field learns from safe and violating queries beyond
% expert trajectories, providing explicit collision-avoidance guidance alongside
% the planner's navigation objective.
\begin{figure}[!htb]
    \centering
    \includegraphics[width=1\linewidth]{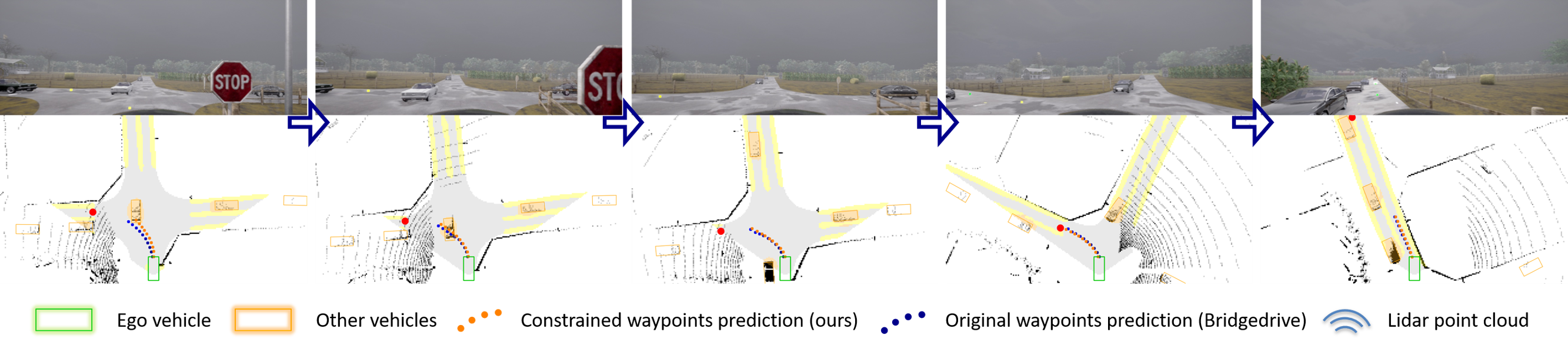}
    \caption{\textbf{Unprotected left turn on Bench2Drive.} BridgeDrive (blue) collides with an oncoming vehicle; BridgeGuard (orange) avoids the collision by bypassing the vehicle. See also Appendix~\ref{app:supplementary-videos}.}
    \label{fig:detail_carla}
\end{figure}
\begin{figure}[!htb]
    \centering
    \includegraphics[width=1\linewidth]{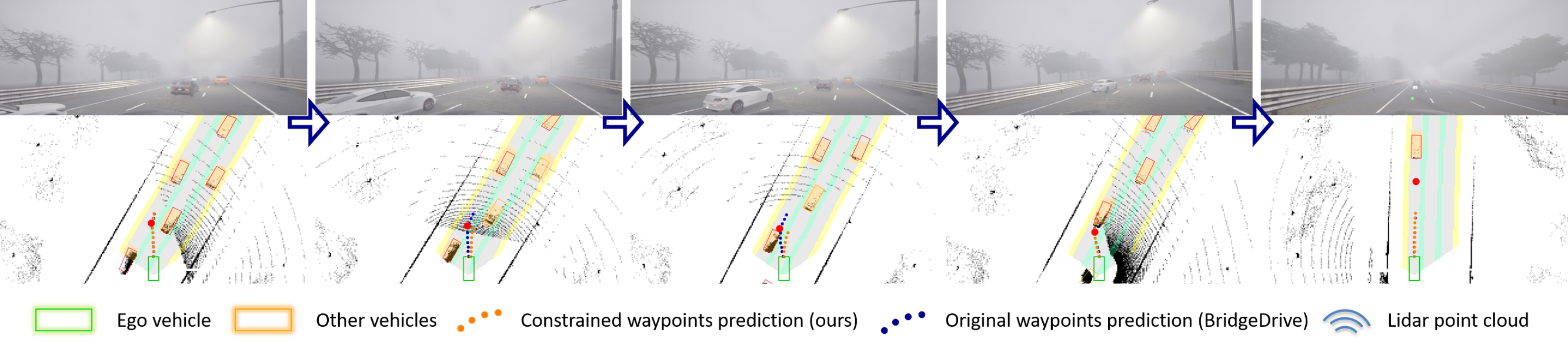}
    \caption{\textbf{Multi-lane change on Bench2Drive.} BridgeDrive (blue) enters an occupied lane; BridgeGuard (orange) yields to the white vehicle and merges safely. See also Appendix~\ref{app:supplementary-videos}.}
    \label{fig:detail_carla2}
\end{figure}

% Related work is included in the appendix.
% \input{base/5_related_work}

\section{LIMITATIONS \& FUTURE WORK}
\label{FUTURE WORK}
BridgeGuard has two main limitations. First, we consider only one distance
function, constructed from eighth-order superellipses. Although it correctly
identifies the prescribed safe set, its gradient can be small in the interior
of a superellipse, providing a weak correction for waypoints that lie deep
inside an unsafe region. 
This behavior may also make the gradient-dominance
condition underlying the idealized $\Omega$-bridge guarantee more difficult to
satisfy. 
Future work should explore alternative smooth distance functions with
well-conditioned gradients throughout unsafe regions and characterize when
they satisfy the sufficient conditions for an $\Omega$-bridge.

Second, our correction is restricted to a low-dimensional parametric curve
space. This promotes geometric coherence but limits the available correction
directions and may not capture the full diversity of drivable, dynamically
feasible trajectories. A promising direction is to learn a richer drivable
trajectory manifold or a constraint-aware trajectory representation that
preserves feasibility while allowing more flexible safety corrections.
% \paragraph{Current limitations and future work.}
% \paragraph{局限性}（1）本文的工作通过在扩散去噪过程引入桥接项，强制生成轨迹进入约束域，从而提升安全性。但这确实加重了模型的计算负担，使得结构更加复杂。我们希望未来能够涉及更加轻量化的距离场预测器（distance-field predictor）。（2）本文提出的预测器依赖比较精确的感知结果。当感知backbone提取的BEV feature出现较大偏差时，会不可避免的出现性能下降。未来我们希望就这两点进行改进，增强算法的鲁棒性，同时减少计算消耗。
% While BridgeGuard improves safety by introducing a bridge drift that enforces trajectory generation within the constraint domain during diffusion denoising, this mechanism inevitably incurs additional computational overhead and increases the complexity of the inference pipeline. Future work could explore a more lightweight distance-field predictor to reduce the computational cost of constraint enforcement. Moreover, the proposed predictor relies on sufficiently accurate perception results, and substantial errors in the BEV features extracted by the perception backbone can lead to performance degradation. We therefore aim to improve robustness against perception inaccuracies while further reducing the computational overhead of BridgeGuard, toward a more efficient and robust constrained trajectory generation framework.
% \subsection{unsolved problems} 
% \begin{enumerate}
%     \item in safety domain $\Omega$, distance functions have to be zero. But now ego vehicle should be far away from the obstacle infinitely .
% \end{enumerate}
\section{Conclusion}
\label{sec:conclusion}
We introduced BridgeGuard, which combines a learned safety field with
low-dimensional curve corrections for diffusion-based driving. DistanceFieldNet
learns values and spatial gradients beyond expert trajectories, drifting
denoising while the perception backbone and planner remain frozen.
On Bench2Drive, BridgeGuard improves DS and SR for both
BridgeDrive and DiffusionDrive$^{\mathrm{geo}}$. Our analysis establishes
sufficient conditions for terminal safety in an idealized continuous-time
bridge; these do not directly guarantee safety for the learned field and
discretized corrections.
\newpage

\subsection*{AI use statement}

We used ChatGPT and Codex to assist in drafting portions of the manuscript,
including some mathematical expressions. ChatGPT also helped improve the
clarity and readability of the text, while Codex assisted with the design of
selected modules and the development of scripts. The authors manually reviewed
all AI-assisted work and checked the correctness of the text, mathematical
expressions, module designs, and scripts. We take full responsibility for the
final content of this work, including all text, mathematical claims, code, and
results produced with the assistance of these tools.

\subsection*{Reproducibility statement}

To support reproducibility, Section~\ref{sec:method} describes our method,
model architecture, and training and inference procedures.
Section~\ref{experiments} reports the datasets, evaluation protocol, and
implementation settings used in our experiments. The assumptions and proofs
underlying our theoretical analysis are provided in
Appendix~\ref{app:omega-bridge-conditions}. We plan to open-source our code and
release our trained model to enable independent reproduction of the reported
results and support further research.

% \input{main.bbl}

% \section*{References}

\bibliography{iclr2027_conference}
\bibliographystyle{apalike}
%%%%%%%%%%%%%%%%%%%%%%%%%%%%%%%%%%%%%%%%%%%%%%%%%%%%%%%%%%%%
\newpage

\appendix
\section{Related Work}
\label{app:related-work}

\subsection{Generative Models for End-to-End Autonomous Driving}
Classical autonomous-driving planners use search, sampling, or optimization with
explicit vehicle models and constraints~\citep{paden2016survey,schwarting2018planning}.
Learning-based approaches instead derive driving policies from demonstrations;
representative systems range from ChauffeurNet, which augments imitation with
synthesized perturbations~\citep{bansal2018chauffeurnetlearningdriveimitating},
to planning-oriented end-to-end architectures that jointly optimize perception,
prediction, and planning~\citep{hu2023planningoriented}. Although these methods
can learn rich scene--action mappings, a single deterministic prediction may
not represent the multiple plausible behaviors in an interactive scene.

Generative planners address this ambiguity by modeling a distribution over
future trajectories. Diffusion models~\citep{ho2020denoising} and flow
matching~\citep{lipman2023flow} provide general frameworks for iterative
conditional generation. In end-to-end driving, GenAD learns a latent space of
multimodal driving behaviors with a variational autoencoder~\citep{zheng2024genad},
whereas DiffusionDrive performs efficient truncated denoising from multimodal
trajectory anchors~\citep{diffusiondrive}.
DiffRefiner uses diffusion to refine coarse proposals through fine-grained
semantic interaction~\citep{yin2025diffrefinercoarsefinetrajectory}. BridgeDrive
instead casts anchor-to-trajectory refinement as a diffusion bridge, allowing
informative endpoints and efficient probability-flow ODE
sampling~\citep{liu2026bridgedrive,zhou2023denoisingdiffusionbridgemodels}.
Recent large-scale studies further investigate diffusion-planner scaling and
real-world deployment~\citep{zheng2026unleash}. BridgeGuard is complementary to
these advances: it retains a pretrained diffusion-bridge planner and adds an
explicit safety-correction mechanism to its reverse process.

\subsection{Safety-Guided and Constrained Trajectory Generation}
Guidance provides a flexible way to control a pretrained diffusion model.
Classifier guidance modifies the score using gradients from an auxiliary
model~\citep{dhariwal2021diffusion}, while classifier-free
guidance combines conditional and unconditional
predictions~\citep{ho2021classifierfree}. In motion generation, differentiable objectives
have similarly been used to guide denoising toward desired behaviors or away
from collisions~\citep{jiang2023motiondiffuser,zhong2023guided}. Such objectives
bias the generated distribution, but their effectiveness depends on the
quality and scale of the guidance gradients; without additional assumptions,
they do not by themselves establish terminal constraint satisfaction.

Other approaches incorporate safety through search, ranking, or constrained
dynamics. Diffusion-ES performs gradient-free reward optimization over
diffusion-generated trajectories~\citep{yang2024diffusiones}, while SafeDrive
uses trajectory-conditioned safety reasoning to rank candidate
plans~\citep{safedrive}. Flow-based planners have
introduced control barrier functions in SafeFlowMatcher~\citep{yang2026safeflowmatcher},
and GuideFlow combines learned energy guidance with physical
constraints~\citep{liu2026guideflow}. Candidate-based methods depend on the coverage of
their proposals, while barrier-based methods require an appropriate barrier
and sufficiently accurate state and geometry estimates.

Constrained-domain diffusion bridges offer a more direct connection between
generation and feasible-set membership~\citep{wu2022diffusion,liu2023learning}.
Manually Bridged Models (MBM) define a distance function that vanishes on a
feasible set $\Omega$ and add its time-weighted negative gradient to the
reverse diffusion dynamics~\citep{naderiparizi2025constrainedgenerativemodelingmanually}.
BridgeGuard adapts this idea to autonomous driving in two ways: it learns a
scene- and horizon-dependent distance field from BEV features, and it applies
the resulting correction in a low-dimensional curve space to preserve path
coherence. We additionally give sufficient conditions under which the
idealized continuous-time manual bridge reaches $\Omega$ at the terminal time.
This guarantee applies to the idealized full-space process; the learned field
and discretized curve-space correction remain practical approximations.

\section{Diffusion Models and Diffusion Bridge Models}
\label{app:diffusion-bridge-background}
This section expands the diffusion-model background in
Section~\ref{sec:diffusion-bridge} and makes explicit how an ordinary diffusion
model extends to a diffusion bridge.

\paragraph{Diffusion models.}
Consider the forward It\^o SDE
\begin{equation}
\mathrm{d}x_t=f(x_t,t)\,\mathrm{d}t+g(t)\,\mathrm{d}w_t,
\label{eq:forward-sde}
\end{equation}
and let $q_t$ denote the marginal density of $x_t$. Its score is
$s(x_t,t)=\nabla_{x_t}\log q_t(x_t)$. Starting from the terminal distribution
$q_T$ and evolving from $T$ to $0$, the reverse-time SDE and the corresponding
probability-flow ODE are, respectively,
\begin{align}
\mathrm{d}x_t
&=\left[f(x_t,t)-g^2(t)s(x_t,t)\right]\mathrm{d}t
+g(t)\,\mathrm{d}\bar w_t,
\label{eq:app-diffusion-reverse-sde}\\
\mathrm{d}x_t
&=\left[f(x_t,t)-\frac{1}{2}g^2(t)s(x_t,t)\right]\mathrm{d}t.
\label{eq:diffusion-pf-ode}
\end{align}
The ODE has the same time-indexed marginal densities as the SDE, but defines a
deterministic transport once its terminal state is fixed.

The score can be learned without evaluating the generally intractable marginal
$q_t$. Given a tractable perturbation kernel $q_{t\mid 0}(x_t\mid x_0)$,
denoising score matching minimizes~\citep{ho2020denoising,zhou2023denoisingdiffusionbridgemodels}
\begin{equation}
\mathcal{L}_{\mathrm{DSM}}(\theta)
=\mathbb{E}_{t,x_0,x_t}\!\left[
w(t)\left\|s_\theta(x_t,t)
-\nabla_{x_t}\log q_{t\mid 0}(x_t\mid x_0)\right\|_2^2
\right],
\label{eq:diffusion-score-matching}
\end{equation}
where $t$ is sampled from $(0,T)$, $x_0$ is a data sample,
$x_t\sim q_{t\mid 0}(\cdot\mid x_0)$, and $w(t)>0$ is a weighting function.
An equivalent parameterization predicts the clean sample by minimizing
\begin{equation}
\mathcal{L}_{x_0}(\theta)
=\mathbb{E}_{t,x_0,x_t}\!\left[
w(t)\left\|x_\theta(x_t,t)-x_0\right\|_2^2
\right].
\label{eq:app-diffusion-loss}
\end{equation}
For the Gaussian kernel
$q_{t\mid 0}(x_t\mid x_0)=\mathcal{N}(\lambda_t x_0,\sigma_t^2I)$,
a clean-sample prediction $\hat x_{0\mid t}=x_\theta(x_t,t)$ yields the score
estimate
\begin{equation}
\hat s_\theta(x_t,t)
=\frac{\lambda_t\hat x_{0\mid t}-x_t}{\sigma_t^2}.
\label{eq:app-endpoint-score}
\end{equation}

\paragraph{Diffusion bridge models.}
Ordinary diffusion models transport samples between data and a simple reference
distribution, usually Gaussian noise. In applications where both endpoints
carry information, a diffusion bridge instead models paired endpoints
$(x_0,x_T)$ and interpolates between their distributions. Conditioning the
reference SDE in Eq.~\eqref{eq:forward-sde} on the terminal event $x_T=y$ gives
the Doob $h$-transform~\citep{doob1984classical,zhou2023denoisingdiffusionbridgemodels}
\begin{equation}
\mathrm{d}x_t
=\left[f(x_t,t)+g^2(t)h(x_t,t;y,T)\right]\mathrm{d}t
+g(t)\,\mathrm{d}w_t,
\qquad
h(x_t,t;y,T)=\nabla_{x_t}\log q_{T\mid t}(y\mid x_t),
\label{eq:bridge-forward-sde}
\end{equation}
where $q_{T\mid t}$ is the transition density of the unconditioned reference
process. The additional drift attracts the process to $y$ as $t\rightarrow T$.

Let
$s^{y}(x_t,t)=\nabla_{x_t}\log q_t(x_t\mid x_T=y)$ denote the score of the
bridge marginal. The reverse-time SDE and its probability-flow ODE are
\begin{align}
\mathrm{d}x_t
&=\left\{f(x_t,t)-g^2(t)
\left[s^{y}(x_t,t)-h(x_t,t;y,T)\right]\right\}\mathrm{d}t
+g(t)\,\mathrm{d}\bar w_t,
\qquad x_T=y,
\label{eq:bridge-reverse-sde}\\
\mathrm{d}x_t
&=\left\{f(x_t,t)-g^2(t)
\left[\frac{1}{2}s^{y}(x_t,t)-h(x_t,t;y,T)\right]\right\}\mathrm{d}t,
\qquad x_T=y.
\label{eq:bridge-pf-ode}
\end{align}
Thus, relative to Eqs.~\eqref{eq:app-diffusion-reverse-sde}
and~\eqref{eq:diffusion-pf-ode}, bridge dynamics retain the terminal endpoint
through both the conditional score and the $h$-transform drift. If terminal
information is discarded and $x_T$ is drawn from a fixed noise distribution,
the formulation reduces to the ordinary diffusion-model setting.

Bridge scores can likewise be learned by denoising bridge score
matching~\citep{zhou2023denoisingdiffusionbridgemodels}. For a tractable bridge
kernel $q_{t\mid 0,T}(x_t\mid x_0,x_T)$, the objective is
\begin{equation}
\mathcal{L}_{\mathrm{BDSM}}(\theta)
=\mathbb{E}_{t,x_0,x_T,x_t}\!\left[
w(t)\left\|s_\theta(x_t,t,x_T)
-\nabla_{x_t}\log q_{t\mid 0,T}(x_t\mid x_0,x_T)\right\|_2^2
\right].
\label{eq:bridge-score-objective}
\end{equation}
For the Gaussian bridge kernel used in Section~\ref{sec:diffusion-bridge},
\begin{equation}
q_{t\mid 0,T}(x_t\mid x_0,x_T)
=\mathcal{N}(a_t x_T+b_t x_0,c_t^2I),
\qquad
\nabla_{x_t}\log q_{t\mid 0,T}
=\frac{a_t x_T+b_t x_0-x_t}{c_t^2}.
\label{eq:app-gaussian-bridge-kernel}
\end{equation}
Consequently, a bridge model may predict the clean endpoint directly via
$\hat x_{0\mid t}=x_\theta(x_t,t,x_T)$ and minimize
\begin{equation}
\mathcal{L}_{\mathrm{bridge}}(\theta)
=\mathbb{E}_{t,x_0,x_T,x_t}\!\left[
w(t)\left\|x_\theta(x_t,t,x_T)-x_0\right\|_2^2
\right],
\label{eq:bridge-loss}
\end{equation}
with induced conditional-score estimate
\begin{equation}
\hat s_\theta(x_t,t,x_T)
=\frac{a_t x_T+b_t\hat x_{0\mid t}-x_t}{c_t^2}.
\label{eq:app-bridge-endpoint-score}
\end{equation}
This endpoint parameterization is the one used by the diffusion-bridge planner
and motivates the score-to-endpoint conversion derived next.

\section{Deriving Endpoint Corrections from Score Corrections}
\label{discussion of correction equation}
Since the diffusion-bridge planner is parameterized by the clean-endpoint
prediction $\hat x_{0\mid t}$ rather than by an explicit score network, we
express the score correction as an equivalent endpoint correction. We first
consider the Gaussian perturbation kernel of a standard diffusion model,
$q_t(x_t\mid x_0)=\mathcal{N}(\lambda_t x_0,\sigma_t^2I)$. Its original and
corrected endpoint predictions induce
\begin{equation}
\hat s_t=\frac{\lambda_t\hat x_{0\mid t}-x_t}{\sigma_t^2},
\qquad
\hat s_t^\Omega
=\frac{\lambda_t\hat x_{0\mid t}^\Omega-x_t}{\sigma_t^2}.
\end{equation}
Evaluating both scores at the same $x_t$ and subtracting gives
\begin{equation}
\Delta s_t
=\hat s_t^\Omega-\hat s_t
=\frac{\lambda_t}{\sigma_t^2}
\left(\hat x_{0\mid t}^\Omega-\hat x_{0\mid t}\right),
\qquad
\hat x_{0\mid t}^\Omega
=\hat x_{0\mid t}+\frac{\sigma_t^2}{\lambda_t}\Delta s_t,
\end{equation}
which is Eq.~\eqref{direct correction}.

We next consider the Gaussian bridge kernel in
Section~\ref{sec:diffusion-bridge},
$q_t(x_t\mid x_0,x_T)=\mathcal{N}(a_t x_T+b_t x_0,c_t^2I)$. With the anchor
$x_T$ and scene context $z$ fixed, the corresponding scores are
\begin{equation}
\hat s_t
=\frac{a_t x_T+b_t\hat x_{0\mid t}-x_t}{c_t^2},
\qquad
\hat s_t^\Omega
=\frac{a_t x_T+b_t\hat x_{0\mid t}^\Omega-x_t}{c_t^2}.
\end{equation}
Subtracting cancels the shared anchor and current-state terms, yielding
\begin{equation}
\Delta s_t
=\frac{b_t}{c_t^2}
\left(\hat x_{0\mid t}^\Omega-\hat x_{0\mid t}\right),
\qquad
\hat x_{0\mid t}^\Omega
=\hat x_{0\mid t}+\frac{c_t^2}{b_t}\Delta s_t.
\end{equation}
Here $\Delta s_t\in\mathbb{R}^{2H}$ is the score increment in the base
planner's waypoint space, induced by resampling the corrected curve. Any
separately predicted target speed is absent from the diffusion state and its score.
The score increment is distinct from the update in curve-parameter space.

A diffusion bridge differs from a conventional diffusion model because its
reverse transition is conditioned not only on the current state $x_t$ but also
on the terminal anchor $x_T$. During one rollout, the selected anchor $x_T$ and
scene context $z$ are fixed conditioning variables. Thus, the reverse dynamics
can be regarded as a conditional Markov process~\citep{stratonovich1965conditional}:
$$
p_{\theta}(x_{t'}\mid x_t,x_T,z).
$$
Under fixed $(x_T,z)$, the bridge score depends on the current state $x_t$ and
diffusion time $t$. This justifies evaluating both the original and corrected
scores at the same $x_t$ when deriving the endpoint correction induced by the
manual bridge term.

\section{Low-dimensional constrained correction}
\label{Low-dimensional constrained correction}
The planner's waypoint sequence $x\in\mathbb{R}^{2H}$ and the temporal
trajectory $P\in\mathbb{R}^{2K}$ used for safety evaluation are connected by
$P=\mathcal{E}_{\boldsymbol{\tau}}(x)$ (Eq.~\eqref{eq:temporal-extrapolation}).
This map is the identity for temporal waypoints at the query horizons and
uses extrapolation for geometric waypoints with a separate target speed.
During correction, the scene $z$, physical horizons $\boldsymbol{\tau}$,
sampling convention, and any auxiliary predictions are fixed. For parameter
space $\Theta\subset\mathbb{R}^{d_\eta}$, the resampling map
$F:\Theta\to\mathbb{R}^{2H}$ follows the base planner's representation and restricts waypoints to
$\mathcal{M}=F(\Theta)$. The fitting map is generally not invertible.

The temporal safety domain $\Omega=\Omega(z;\boldsymbol{\tau})$ induces
safe sets for planner waypoints and curve parameters:
\begin{equation}
\begin{aligned}
\Omega_{\mathrm{plan}}
&=\{x\in\mathbb{R}^{2H}:
\mathcal{E}_{\boldsymbol{\tau}}(x)\in\Omega\},\\
\Theta^\Omega
&=\{\boldsymbol{\eta}\in\Theta:
\mathcal{E}_{\boldsymbol{\tau}}(F(\boldsymbol{\eta}))\in\Omega\}.
\end{aligned}
\end{equation}
Consequently, $F(\Theta^\Omega)=\mathcal{M}\cap\Omega_{\mathrm{plan}}$.
Curve refinement searches only among representable waypoint sequences whose
temporal query positions satisfy the safety constraints.

The learned temporal-trajectory loss is
$\widehat\ell_\psi^\Omega(P;z,\boldsymbol{\tau})
=\sum_{i=1}^{K}D^\Omega_\psi(\mathbf{p}_i,\tau_i;z)$.
The curve objective composes resampling with evaluation at physical horizons:
\begin{equation}
\widehat{\mathcal{\ell}}_\psi^\Omega
(\boldsymbol{\eta};z,\boldsymbol{\tau})
=\widehat\ell_\psi^\Omega\!\left(
\mathcal{E}_{\boldsymbol{\tau}}(F(\boldsymbol{\eta}));
z,\boldsymbol{\tau}\right).
\end{equation}
Let $J_F=\partial F/\partial\boldsymbol{\eta}$ and
$J_{\mathcal{E}}=\partial\mathcal{E}_{\boldsymbol{\tau}}(x)/\partial x$.
At differentiable points, the update in Eq.~\eqref{eq:curve-correction}
induces displacements in planner waypoints and temporal query positions,
$\Delta x=F(\boldsymbol{\eta}^\Omega)-F(\boldsymbol{\eta})$ and
$\Delta P=\mathcal{E}_{\boldsymbol{\tau}}(F(\boldsymbol{\eta}^\Omega))-P$:
\begin{equation}
\begin{aligned}
\nabla_{\boldsymbol{\eta}}\widehat{\mathcal{\ell}}_\psi^\Omega
&=J_F^\top J_{\mathcal{E}}^\top\nabla_P\widehat\ell_\psi^\Omega,\\
\Delta x
&=-\alpha(t)J_FW_\eta J_F^\top J_{\mathcal{E}}^\top
\nabla_P\widehat\ell_\psi^\Omega+o(\alpha(t)),\\
\Delta P
&=-\alpha(t)J_{\mathcal{E}}J_FW_\eta J_F^\top J_{\mathcal{E}}^\top
\nabla_P\widehat\ell_\psi^\Omega+o(\alpha(t)).
\end{aligned}
\label{eq:curve-pullback}
\end{equation}
All derivatives are evaluated at the current parameters,
$x=F(\boldsymbol{\eta})$, and $P=\mathcal{E}_{\boldsymbol{\tau}}(x)$.
For temporal waypoints at the query horizons, $J_{\mathcal{E}}=I$ and
$\Delta P=\Delta x$; otherwise the safety gradient also propagates through
$\mathcal{E}_{\boldsymbol{\tau}}$. With positive diagonal $W_\eta$, the parameter update
is a descent direction whenever this pulled-back gradient is nonzero;
whether a finite step decreases the objective depends on its magnitude.
This construction restricts waypoint changes to the tangent space of
$\mathcal{M}$ and temporal query changes to their images under
$\mathcal{E}_{\boldsymbol{\tau}}$. It is an approximation
to direct manual-bridge correction in the full temporal-waypoint space.

\section{Query-conditioned Distance Field Network}\label{neutral network}
DistanceFieldNet takes a waypoint sequence $P=(\mathbf{p}_1,\ldots,\mathbf{p}_K)$,
the corresponding physical horizons $\boldsymbol{\tau}=(\tau_1,\ldots,\tau_K)$,
and BEV features $\mathbf{F}_{\mathrm{BEV}}(z)$ as input. It predicts a scalar
distance-function value for each waypoint--horizon pair and sums these values
to obtain the trajectory-level output. The BEV features are first encoded by
a shared convolutional stem,
$$
\mathbf{F}_0=\mathrm{Conv}_{\psi}(\mathbf{F}_{\mathrm{BEV}}),
$$
which provides a base spatial representation shared across all $K$ queries.
We describe the per-waypoint computation for one query pair
$(\mathbf{p},\tau)=(\mathbf{p}_i,\tau_i)$, omitting the waypoint index below.
This representation is conditioned on both the queried spatial location and
the queried prediction time.

The query point ${\mathbf{p}}$ is normalized by the BEV metric extent and encoded with sinusoidal positional features:
$$
\bar{\mathbf{p}}
=
\left(
\frac{x}{x_{\max}},
\frac{y}{y_{\max}}
\right),
\qquad
\mathbf{e}_{\mathbf{p}}
=
\mathrm{MLP}_{\psi}
\left(
\left[
\bar{\mathbf{p}},
\mathrm{PE}(\bar{\mathbf{p}})
\right]
\right).
$$
The query token attends to dense BEV tokens through a stack of cross-attention blocks:
$$
\mathbf{h}
=
\mathrm{CrossAttnBlocks}
\left(
\mathbf{e}_{\mathbf{p}},
\mathbf{Z}_{\mathrm{BEV}}
\right),
$$
where $\mathbf{Z}_{\mathrm{BEV}}$ denotes the BEV tokens with spatial positional encoding. The resulting query context produces spatial FiLM~\citep{perez2017filmvisualreasoninggeneral} parameters:
$$
(\boldsymbol{\gamma}_{\mathbf{p}},\boldsymbol{\beta}_{\mathbf{p}})
=
\mathrm{SpatialFiLM}_{\psi}(\mathbf{h}).
$$
These parameters condition the BEV feature on the queried coordinate:
$$
\mathbf{F}_{\mathbf{p}}
=
\boldsymbol{\gamma}_{\mathbf{p}}\odot \mathbf{F}_0
+
\boldsymbol{\beta}_{\mathbf{p}}.
$$

The prediction time $\tau$ is encoded by a sinusoidal embedding followed by an MLP, and transformed into temporal FiLM parameters:
$$
\mathbf{e}_{\tau}=\mathrm{MLP}_{\tau}(\mathrm{PE}(\tau)),
\qquad
(\boldsymbol{\gamma}_{\tau},\boldsymbol{\beta}_{\tau})
=
\mathrm{TemporalFiLM}_{\psi}(\mathbf{e}_{\tau}).
$$
The temporal FiLM injects future-time information into the query-conditioned BEV feature:
$$
\mathbf{F}_{\tau,\mathbf{p}}
=
\boldsymbol{\gamma}_{\tau}\odot \mathbf{F}_{\mathbf{p}}
+
\boldsymbol{\beta}_{\tau}.
$$

The essential design is to condition the BEV representation on both query location and prediction time before distance-field prediction.
%, and the method does not rely on a strict semantic priority between the two conditioning operations.
We further compute a query-specific gate:
$$
G
=
\mathbf{G}_{\psi}
\left(
\mathbf{Z}_{\mathrm{BEV}},\mathbf{h},\mathbf{F}_{\tau,\mathbf{p}}
\right).
$$
Together with a learned convolutional gate, the conditioned BEV feature is filtered as
$$
\mathbf{F}_{g}
=
\mathbf{F}_{\tau,\mathbf{p}}
\odot
G.
$$

Finally, the gated feature is passed to a DistanceFieldHead, which denotes the complete output mapping from the query-conditioned BEV feature to a scalar distance-function value. It consists of three components: a BEV aggregation head that integrates spatial BEV evidence into a dense response map, a lightweight reducer that gradually reduces the flattened response map to a scalar logit, and a bounded monotonic activation:
$$
\mathbf{R}
=
\mathrm{Aggregation}_{\psi}(\mathbf{F}_{g}),
\qquad
r
=
\mathrm{Reducer}_{\psi}(\mathrm{Flatten}(\mathbf{R})).
$$
The final distance-function value is then obtained by the activation part of DistanceFieldHead:
$$
D_\psi(\mathbf{p},\tau,\mathbf{F}_{\mathrm{BEV}})
=
\frac{\pi}{2}-\arctan(\operatorname{softplus}(r)).
$$

The same network parameters are used for all waypoint--horizon pairs.
Summing the per-waypoint predictions gives the trajectory-level safety objective:
\begin{equation}
\widehat\ell_\psi^\Omega(P;z,\boldsymbol{\tau})
=\sum_{i=1}^{K}
D_\psi(\mathbf{p}_i,\tau_i,\mathbf{F}_{\mathrm{BEV}}(z)).
\end{equation}
Here $D_\psi(\mathbf{p},\tau,\mathbf{F}_{\mathrm{BEV}}(z))$ implements the
scene-conditioned field $D_\psi^\Omega(\mathbf{p},\tau;z)$ used in the method.
Differentiating the summed objective through the waypoint coordinates and
trajectory parameterization provides the curve-space correction direction.

\section{Implementation Configuration}
\label{app:implementation-config}
Tables~\ref{tab:distancefield-config} and~\ref{tab:training-sensor-config}
summarize the DistanceFieldNet architecture and the training and sensor
settings used in our BridgeDrive-based implementation. Network operations
are described in Appendix~\ref{neutral network}.

\paragraph{Planner and control.}
The pretrained perception backbone and diffusion planner remain frozen during
DistanceFieldNet training. The planner outputs ten geometric waypoints in the
ego frame ($H=10$) and a scalar target speed. Safety injection refines the path
while keeping the target speed fixed, and a PID controller generates throttle
and steering commands. All remaining decoding settings follow the respective
pretrained diffusion planner.

\paragraph{Safety labels and query horizons.}
Training labels use ground-truth obstacle half-lengths $a_j$ and half-widths
$b_j$ without additional safety margins. Obstacle positions are linearly
extrapolated using their annotated velocities, while headings remain constant.
We set $\kappa=1$, so each query is supervised by its largest per-obstacle
violation. Physical query times are uniformly spaced over a horizon of 0.5~s.
% TODO: Resolve the reported K=10 versus eight temporal query positions,
% and specify whether the temporal grid includes the current time.
At inference, the field is predicted from observed BEV features and
waypoint--horizon queries, without explicit obstacle dimensions or
ground-truth geometry.

\paragraph{Curve refinement.}
We fit denoised waypoints by ridge regression using
$y(x)=\eta_1x+\eta_2x^2+\eta_4x^4$, where $x$ is the forward coordinate and
$y$ is the lateral displacement in the ego frame. Thus,
$\boldsymbol{\eta}=(\eta_1,\eta_2,\eta_4)$ has dimension $d_\eta=3$, with
zero constant and cubic terms. The quadratic ablation omits the quartic term,
giving $d_\eta=2$. Corrected curves are resampled at approximately 1~m
arc-length intervals. We use $\gamma_\epsilon(t)=1/(t+\epsilon)$ with
$\epsilon>0$ and $W_\eta=10^{-3}I_{d_\eta}$ in
Eq.~\eqref{eq:curve-correction}, including the corresponding score-to-endpoint
conversion factor.
% TODO: Report the ridge regularization coefficient and numerical epsilon.

% Draft tables: highlighted entries require confirmation from the implementation.

\begin{table}[htbp]
\centering
\caption{DistanceFieldNet configuration ($B$: batch size, $Q$: query number).}
\label{tab:distancefield-config}
\small
\setlength{\tabcolsep}{5pt}
\renewcommand{\arraystretch}{1.12}
\begin{tabular}{@{}p{0.38\linewidth}p{0.58\linewidth}@{}}
\toprule
Component & Configuration \\
\midrule
Inputs & Shared BEV features, waypoint coordinates, and corresponding time horizons \\
BEV feature shape & $(B, 64, 64, 64)$ \\
Convolutional stem & $\mathbf{F}_{\mathrm{BEV}}\in\mathbb{R}^{B\times 64\times H\times W}
\rightarrow
\mathbf{F}_0\in\mathbb{R}^{B\times 32\times H\times W}$ \\
Spatial query encoding & Sinusoidal positional encoding and MLP; dimensions = 256 \\
Temporal query encoding & Sinusoidal positional encoding and MLP; dimensions = 256 \\
Query normalization & Position scaled by BEV metric extent ($y_{\max}=x_{\max}=32$)\\
Cross-attention & Query:
$\mathbf{z}_q\in\mathbb{R}^{BQ\times1\times256}$;
Key/Value:
$\mathbf{Z}_{\mathrm{BEV}}\in\mathbb{R}^{BQ\times HW\times256}$\\
Feature modulation & $\mathbf{F}'=\mathbf{F}\odot(1+\boldsymbol{\gamma})+\boldsymbol{\beta}$ \\
Gate & 
$\operatorname{Conv}_{3\times3}(32,16)$, GN$(4)$, GELU,
$\operatorname{Conv}_{3\times3}(16,1)$, Sigmoid.
\\
DistanceFieldHead
&
$\mathbf{F}_{q}^{g}\in\mathbb{R}^{BQ\times32\times H\times W}
\rightarrow
\mathbf{R}_q\in\mathbb{R}^{BQ\times1\times H\times W}$\\
Reducer & $\mathbf{R}_q\in\mathbb{R}^{BQ\times1\times H\times W}
\rightarrow r_q\in\mathbb{R}^{BQ}$ \\
Output activation & $\pi/2-\arctan(\operatorname{softplus}(r_q))$ \\
\bottomrule
\end{tabular}
\end{table}

\begin{table}[htbp]
\centering
\caption{Training and sensor configuration.}
\label{tab:training-sensor-config}
\small
\setlength{\tabcolsep}{5pt}
\renewcommand{\arraystretch}{1.12}
\begin{tabular}{@{}p{0.38\linewidth}p{0.58\linewidth}@{}}
\toprule
Setting & Configuration \\
\midrule
Trainable module & DistanceFieldNet only; perception backbone and planner frozen \\
Optimizer & AdamW \\
Initial learning rate / epochs & $3\times10^{-4}$ / 30 \\
Batch size & 50 (per GPU) \\
% Weight decay / optimizer betas & \hl{TBD} \\
Learning-rate schedule & cosine annealing schedule \\
Gradient clipping & none \\
Supervision & Field values and spatial gradients, equally weighted ($\lambda_{\mathrm{gd}}=1$) \\
Anchor bank & 60 trajectory anchors adapted from BridgeDrive for mode classification and training-query construction \\
Training queries per scene & 20 spatial points randomly sampled from two randomly selected anchors \\
Range observations & 4 LiDAR frames \\
% Range-frame count / spacing & \hl{TBD} \\
Motion compensation & Constant Velocity Prediction \\
Camera input & One front-view image ($256 \times 256$) \\
Navigation input & Target-point coordinates \\
Hardware & Eight NVIDIA H20 GPUs, 96~GB each \\
Training cost & Approximately 190 GPU hours \\
\bottomrule
\end{tabular}
\end{table}

\section{Additional Experimental Results}
\label{app:additional-results}

\subsection{Curve Dimensionality: Full Results}
\label{app:curve-dimensionality}

Table~\ref{tab:curve-dimensionality-full} reports the full Bench2Drive results
for the quadratic ablation ($\eta_4=0$) and the quartic curve representation.
The quartic representation improves four of the five ability metrics, with
the largest gains in giving way (20.00\% to 30.00\%) and overtaking
(60.00\% to 68.89\%). Although merging performance decreases slightly from
75.00\% to 73.75\%, mean ability increases from 64.77\% to 69.74\%, alongside
improvements in DS and SR. These results support retaining the quartic term
for safety correction across diverse driving scenarios.

\begin{table}[htbp]
\centering
\caption{Full curve-dimensionality results on Bench2Drive.}
\label{tab:curve-dimensionality-full}
\small
\setlength{\tabcolsep}{4pt}
\begin{tabular}{@{}lrr@{}}
\toprule
Metric & Quadratic & Quartic \\
\midrule
DS & 88.65 & \textbf{90.88} \\
SR (\%) & 73.18 & \textbf{76.36} \\
\midrule
Merging (\%) & \textbf{75.00} & 73.75 \\
Overtaking (\%) & 60.00 & \textbf{68.89} \\
Emer. brake (\%) & 78.33 & \textbf{85.00} \\
Give way (\%) & 20.00 & \textbf{30.00} \\
Traffic sign (\%) & 90.53 & \textbf{91.05} \\
Mean ability (\%) & 64.77 & \textbf{69.74} \\
\bottomrule
\end{tabular}
\end{table}

\subsection{Evaluation with LEAD Training}
\label{app:lead-results}

\begin{table}[ht]
\centering
\caption{Performance adapted to LEAD on Bench2Drive.}
\label{tab:sota_comparison on LEAD}
\resizebox{\columnwidth}{!}{%
\begin{tabular}{lccccc}
\toprule
Method & Expert & DS & SR(\%) & Effi. & Comfort. \\
\midrule
TFv6~\citep{Nguyen2026CVPR} & LEAD &
$95.2\pm0.3$ &
$86.8\pm0.7$ &
N/A &
N/A \\
BridgeDrive~\citep{liu2026bridgedrive} & LEAD &
$96.34\pm0.55$ &
$89.25\pm0.50$ &
$202.92\pm3.27$ &
$\mathbf{23.24\pm1.06}$ \\
\rowcolor{blue!10}
BridgeGuard(Ours)  & LEAD &
$\mathbf{96.76\pm 0.15}$ &
$\mathbf{89.78\pm 0.32}$ &
$\mathbf{208.50\pm 2.04}$ &
${20.37}\pm$0.19 \\
\bottomrule
\end{tabular}%
}
\end{table}
% \begin{table}[ht]
% \centering
% \caption{Performance adapted to LEAD on Bench2Drive.}
% \label{tab:sota_comparison on LEAD}
% \resizebox{\columnwidth}{!}{%
% \begin{tabular}{lccccc}
% \toprule
% Method & Expert & DS & SR(\%) & Effi. & Comfort. \\
% \midrule
% TFv6~\citep{Nguyen2026CVPR} & LEAD &
% $95.2$ &
% $86.8$ &
% N/A &
% N/A \\
% BridgeDrive~\citep{liu2026bridgedrive} & LEAD &
% $96.34$ &
% $89.25$ &
% $202.92$ &
% $23.24$ \\
% \rowcolor{blue!10}
% BridgeGuard(Ours)  & LEAD &
% $\mathbf{96.65}$&
% $\mathbf{90.00}$ &
% $\mathbf{207.06}$ &
% $\mathbf{20.50}$\\
% \bottomrule
% \end{tabular}%
% }
% \end{table}
% 表格~\ref{tab:sota_comparison on LEAD}显示了我们的方法由数据集LEAD训练后在Bench2Drive基准上的评测结果，同样取得了sota水平。注意到多条基线在该数据集上训练的实验评测结果接近，本文方法的提升幅度也并不明显。这是因为在该数据集上训练后，相关基线的安全失效情形占比已经大幅减小。
Table~\ref{tab:sota_comparison on LEAD} presents the evaluation results of our method on the Bench2Drive benchmark when trained with the LEAD dataset, where BridgeGuard achieves state-of-the-art performance. Compared with training on the original dataset, the performance gap between different baselines becomes noticeably smaller under LEAD. We attribute this reduced margin to the higher quality of the LEAD dataset, which substantially alleviates safety-critical failure cases for existing baselines and consequently leaves less room for further improvement.

\section{Supplementary Videos}
\label{app:supplementary-videos}
We provide six videos in the supplementary materials comparing BridgeDrive
and BridgeGuard (ours) in three scenarios: two lane-changing cases and one
unprotected left-turn case. Each scenario includes one video for each model.
We refer readers to these paired videos for a qualitative comparison of the
models' driving behavior and interactions with surrounding traffic.

\section{Sufficient Conditions for \texorpdfstring{$\Omega$}{Omega}-Bridge}
\label{app:omega-bridge-conditions}

This section specializes the sufficient conditions in Appendix~C.1 of
\cite{naderiparizi2025constrainedgenerativemodelingmanually} to the
autonomous-driving setting.

\subsection{Original sufficient condition}
\label{app:original-bridge-condition}

Let $X_t\in\mathbb{R}^d$ denote the stochastic process generated by the  manually bridged model~\ref{sec:manual-bridges}
\begin{equation}
\mathrm{d}x_t
=\left[\nu(x_t,t)
-g^2(t)b^{\Omega}(x_t,t)\right]\mathrm{d}t
+g(t)\mathrm{d}\bar{w}_t
\label{eq:manually-bridged-model}
\end{equation}
and let $p_t^\Omega$ be its marginal law. Following the notation of \cite{naderiparizi2025constrainedgenerativemodelingmanually}, define
\begin{align}
\beta(t)
&:=
\mathbb{E}_{X_t\sim p_t^\Omega}
\!\left[
\nabla_x\ell^\Omega(X_t;t)^\top\nu(X_t,t)
\right],
\label{eq:app-original-beta}\\
\rho(t)
&:=
\mathbb{E}_{X_t\sim p_t^\Omega}
\!\left[
\partial_t\ell^\Omega(X_t;t)
+\frac{1}{2}g^2(t)
\operatorname{tr}\!\left(\nabla_x^2\ell^\Omega(X_t;t)\right)
\right],
\label{eq:app-original-rho}\\
\zeta(t)
&:=
\exp\!\left(\int_t^T \gamma(s)\,\mathrm{d}s\right).
\label{eq:app-original-zeta}
\end{align}
The two sufficient conditions in Appendix~C.1 of
\cite{naderiparizi2025constrainedgenerativemodelingmanually} are as
follows.
\begin{enumerate}
    \item The $\Omega$-distance satisfies the expected
    Polyak--\L{}ojasiewicz condition, at every fixed diffusion time,
    \begin{equation}
    \mathbb{E}_{X_t\sim p_t^\Omega}
    \!\left[\ell^\Omega(X_t;t)\right]
    \leq
    \mathbb{E}_{X_t\sim p_t^\Omega}
    \!\left[\|\nabla_x\ell^\Omega(X_t;t)\|^2\right],
    \qquad t\in[0,T].
    \label{eq:app-expected-pl-untouched}
    \end{equation}

    \item The integrating factor dominates both the initial value and the
    accumulated drift, time-variation, and diffusion-curvature terms:
    \begin{equation}
    \lim_{t\downarrow0}\zeta(t)=+\infty,
    \qquad
    \lim_{t\downarrow0}
    \frac{\zeta(t)}{
    \displaystyle\int_t^T
    \zeta(s)\bigl(\beta(s)+\rho(s)\bigr)\,\mathrm{d}s}
    =+\infty.
    \label{eq:app-original-asymptotic-condition}
    \end{equation}
\end{enumerate}
We leave Condition~1 in Eq.~\eqref{eq:app-expected-pl-untouched} unchanged. The remainder of this section replaces only the less
transparent asymptotic Condition~2 by assumptions with a direct interpretation
for trajectory planning.

\subsection{A terminal-domination condition for trajectory planning}
\label{app:terminal-domination-condition}

Fix a terminal interval $(0,\delta]$. Assume
that there are constants $G_\Omega,H_\Omega,M_\nu,\overline g^{\,2},\kappa>0$
and an exponent $q\in[0,1)$ such that the following named assumptions hold
for $0<t\leq \delta$:
\begin{align}
\intertext{\textbf{(A1) Time-invariant distance function:}}
\ell^\Omega(x;t)
&=\ell^\Omega(x),
&
\label{eq:app-time-independent-loss}\\
\intertext{\textbf{(A2) Bounded gradient and curvature of the distance function:}}
\|\nabla_x\ell^\Omega(x)\|
&\leq G_\Omega,
&
\|\nabla_x^2\ell^\Omega(x)\|_{\mathrm{op}}
&\leq H_\Omega,
\label{eq:app-loss-regularity-bounds}\\
\intertext{\textbf{(A3) Controlled non-safety drift and diffusion terms:}}
\mathbb{E}_{t, X_t\sim p_t^\Omega}
\!\left[\|\nu(X_t,t)\|\right]
&\leq M_\nu t^{-q},
&
g^2(t)&\leq\overline g^{\,2},
\label{eq:app-drift-diffusion-bounds}\\
\intertext{\textbf{(A4) Sufficiently singular terminal guidance:}}
\gamma(t)
&\geq \frac{\kappa}{t},
&&
\label{eq:app-terminal-attraction}\\
\intertext{\textbf{(A5) Well-behaved coefficients away from the terminal endpoint:}}
\gamma,\beta,\rho
&\in L^1([\varepsilon,T]),
\quad \text{for every }\varepsilon\in(0,T).
&&
\label{eq:app-local-integrability}
\end{align}

\begin{Proposition}[Terminal domination implies the original Condition~2]
\label{prop:app-terminal-domination}
Under Assumptions A1--A5 in Eqs.~\eqref{eq:app-time-independent-loss}--
\eqref{eq:app-local-integrability}, we have
\begin{equation}
\lim_{t\downarrow0}\zeta(t)=+\infty,
\qquad
\lim_{t\downarrow0}
\frac{1}{\zeta(t)}
\int_t^T\zeta(s)
\bigl(|\beta(s)|+|\rho(s)|\bigr)\,\mathrm{d}s=0.
\label{eq:app-strong-asymptotic-condition}
\end{equation}
\end{Proposition}

\begin{proof}
Time independence of the safety potential gives
\begin{equation}
|\rho(t)|
\leq
\frac{d}{2}\,\overline g^{\,2}H_\Omega.
\label{eq:app-rho-bound}
\end{equation}
The Cauchy--Schwarz inequality for the Euclidean inner product and the first
moment bound on $\nu$ give
\begin{equation}
|\beta(t)|
\leq
\mathbb{E}
\!\left[\|\nabla_x\ell^\Omega(X_t)\|\,\|\nu(X_t,t)\|\right]
\leq G_\Omega M_\nu t^{-q}.
\label{eq:app-beta-bound}
\end{equation}
It follows that, for some finite constant $C_R$,
\begin{equation}
R(t)=|\beta(t)|+|\rho(t)|\leq C_Rt^{-q},
\qquad 0<t\leq \delta.
\label{eq:app-forcing-bound}
\end{equation}

By Eqs.~\eqref{eq:app-terminal-attraction} and
\eqref{eq:app-local-integrability}, let
\begin{equation}
C_\delta
:=
\exp\!\left(\int_\delta^T\gamma(r)\,\mathrm{d}r\right)
\in(0,\infty).
\end{equation}
Then, for $0<t\leq\delta$,
\begin{align}
\zeta(t)
&=
C_\delta
\exp\!\left(\int_t^\delta\gamma(r)\,\mathrm{d}r\right)\\
&\geq
C_\delta
\exp\!\left(\int_t^\delta\frac{\kappa}{r}\,\mathrm{d}r\right)
=
C_\delta\left(\frac{\delta}{t}\right)^\kappa.
\label{eq:app-zeta-lower-bound}
\end{align}
so $\zeta(t)\to+\infty$. Moreover, whenever $t\leq s\leq\delta$,
\begin{equation}
\frac{\zeta(s)}{\zeta(t)}
=
\exp\!\left(-\int_t^s\gamma(r)\,\mathrm{d}r\right)
\leq
\left(\frac{t}{s}\right)^\kappa.
\label{eq:app-zeta-ratio}
\end{equation}
Therefore,
\begin{align}
\frac{1}{\zeta(t)}
\int_t^{\delta}\zeta(s)R(s)\,\mathrm{d}s
&\leq
C_Rt^\kappa\int_t^{\delta}s^{-(\kappa+q)}\,\mathrm{d}s
\nonumber\\
&=
\begin{cases}
O(t^\kappa), & \kappa+q<1,\\
O\!\left(t^\kappa\log(\delta/t)\right), & \kappa+q=1,\\
O(t^{1-q}), & \kappa+q>1.
\end{cases}
\label{eq:app-terminal-integral-rates}
\end{align}
All three cases converge to zero because $\kappa>0$ and $q<1$. The integral
over $[\delta,T]$ is finite and independent of the lower limit $t$; after
division by $\zeta(t)\to\infty$, its contribution also vanishes. This proves
Eq.~\eqref{eq:app-strong-asymptotic-condition}. Finally,
\begin{equation}
\left|
\int_t^T\zeta(s)(\beta(s)+\rho(s))\,\mathrm{d}s
\right|
\leq
\int_t^T\zeta(s)R(s)\,\mathrm{d}s,
\end{equation}
which proves the claimed implication.
\end{proof}

\begin{Corollary}[The manual bridge is an $\Omega$-bridge]
\label{cor:app-manual-omega-bridge}
Suppose that
Condition~1 in Eq.~\eqref{eq:app-expected-pl-untouched} and Assumptions
A1--A5 in Eqs.~\eqref{eq:app-time-independent-loss}--
\eqref{eq:app-local-integrability} hold. Then the manual bridge term
$b^\Omega$ in Eq.~\eqref{eq:manually-bridged-model} is an $\Omega$-bridge;
equivalently, the solution of the manually bridged reverse SDE satisfies
\begin{equation}
\mathbb{P}\!\left(X_0\in\Omega\right)=1.
\label{eq:app-terminal-membership}
\end{equation}
\end{Corollary}

\subsection{Interpretation of Condition~1 and Assumption~A3}
\label{app:interpretation-difficult-conditions}

Condition~1 and Assumption~A3 are the only two conditions that can present difficulty in verification because they may involve calculating an expectation under the marginal distribution of the reverse SDE. Let us give some interpretation on these conditions.

Condition~1 is a distributional gradient-dominance requirement. At each fixed
diffusion time, define
\begin{equation}
L_t:=\mathbb{E}_{p_t^\Omega}[\ell^\Omega(X_t;t)],
\qquad
D_t:=\mathbb{E}_{p_t^\Omega}
\!\left[\|\nabla_x\ell^\Omega(X_t;t)\|^2\right].
\end{equation}
Then Eq.~\eqref{eq:app-expected-pl-untouched} requires $L_t\leq D_t$.
A stronger, easier-to-check sufficient condition is the pointwise inequality
$\ell^\Omega(x;t)\leq\|\nabla_x\ell^\Omega(x;t)\|^2$ on the states reachable
by the process. More generally, a finite constant in
$\ell^\Omega\leq C\|\nabla_x\ell^\Omega\|^2$ can be absorbed by rescaling the
potential and inversely rescaling its guidance coefficient.

However, no non-zero, globally bounded $C^1$ distance function can
satisfy a pointwise PL inequality on all of $\mathbb{R}^d$ because the gradient of $\ell^\Omega(x)$ vanishes at maximum point. The expected condition can nevertheless
hold if $p_t^\Omega$ assigns sufficiently little mass to such regions. 
% A useful
% empirical diagnostic is the ratio $D_t/L_t$ at a grid of diffusion times for
% which $L_t>0$; values bounded below by one support the assumption, although a
% finite-sample check is not a formal certificate.

Assumption~A3 controls the drift of the original reverse dynamics along its
actual marginal law. Its drift requirement is equivalently
\begin{equation}
\sup_{0<t\leq\delta}
t^q\,\mathbb{E}_{X_t\sim p_t^\Omega}
\!\left[\|\nu(X_t,t)\|\right]<\infty
\qquad\text{for some }q<1.
\label{eq:app-drift-moment-interpretation}
\end{equation}
% The restriction $q<1$ makes this possible endpoint growth integrable and is
% what allows the forcing bound in Eq.~\eqref{eq:app-forcing-bound} to vanish
% after normalization by $\zeta$. Importantly, a bound on
% $\|\mathbb{E}[\nu(X_t,t)]\|$ is insufficient because cancellations in the mean
% do not control $\mathbb{E}\|\nu(X_t,t)\|$. A convenient stronger condition is
% \begin{equation}
% \mathbb{E}\!\left[\|\nu(X_t,t)\|^2\right]
% \leq M_2^2t^{-2q},
% \end{equation}
% which implies A3 by Cauchy--Schwarz. Analytically, one may also decompose
% $\nu$ into its base-drift, score, and endpoint-conditioning terms and bound
% their first moments separately. Empirically, plotting
% $\mathbb{E}\|\nu(X_t,t)\|$ against $t$ on log--log axes provides a diagnostic
% for whether its endpoint growth is slower than $t^{-1}$. 
The remaining part of
A3, boundedness of $g^2$, follows directly from the chosen analytic diffusion
schedule in both VP and VE cases~\cite{zhou2023denoisingdiffusionbridgemodels}.

% \subsection{Interpretation for VP and VE schedules}
% \label{app:terminal-domination-interpretation}

% Following the notation of Denoising Diffusion Bridge Models
% (DDBM)~\cite{zhou2023denoisingdiffusionbridgemodels}, let $\alpha_t$ and
% $\sigma_t$ denote the signal and noise schedules. Its VP instantiation uses
% \begin{align}
% \mathbf{f}(\mathbf{x}_t,t)
% &=
% \frac{\mathrm{d}\log\alpha_t}{\mathrm{d}t}\mathbf{x}_t,
% &
% g^2(t)
% &=
% \frac{\mathrm{d}}{\mathrm{d}t}\sigma_t^2
% -2\frac{\mathrm{d}\log\alpha_t}{\mathrm{d}t}\sigma_t^2,
% \label{eq:app-vp-schedule}
% \end{align}
% whereas its VE instantiation sets
% \begin{equation}
% \alpha_t=1,
% \qquad
% \mathbf{f}(\mathbf{x}_t,t)=\mathbf{0},
% \qquad
% g^2(t)=\frac{\mathrm{d}}{\mathrm{d}t}\sigma_t^2.
% \label{eq:app-ve-schedule}
% \end{equation}
% For the experimental schedules in DDBM, the VE choice $\sigma_t=t$ gives
% $g^2(t)=2t$, while the VP choice uses the constant rate
% $\beta_1=\beta_0$, giving $g^2(t)=\beta_0$. Both diffusion coefficients are
% bounded on a finite time interval and therefore satisfy the diffusion bound in
% Eq.~\eqref{eq:app-drift-diffusion-bounds}.

% \input{base/checklist}

\end{document}